\documentclass[11pt]{article} 

\newcommand{\bibpath}{bib}

\usepackage{jmlr2e_cus}
\usepackage[usenames]{xcolor}
\usepackage[bookmarksnumbered=true]{hyperref}
\usepackage{url}
\definecolor{cite_color}{RGB}{0, 0, 255}
\definecolor{link_color}{RGB}{153, 0,0}  
\definecolor{url_color}{RGB}{153, 102,  0}
\definecolor{emp_color}{RGB}{0,0,255}
\hypersetup{
 colorlinks,
 citecolor=cite_color,
 linkcolor=link_color,
 urlcolor=url_color}

\usepackage{epsfig}
\usepackage{amsmath}
\usepackage{amssymb}
\usepackage{mathrsfs}
\usepackage{graphicx}
\usepackage[
]{subfig}
\usepackage{multirow}
\usepackage{booktabs}
\usepackage{wrapfig}
\usepackage{enumerate}
\usepackage{bm}
\usepackage{tabularx}

\graphicspath{{./figs/}}

\usepackage[numbers]{natbib}
\bibpunct{[}{]}{;}{n}{,}{,}
\renewcommand{\citet}{\cite}

\usepackage{mathtools}

\usepackage[vlined, ruled, linesnumbered]{algorithm2e}

\SetCommentSty{mycommfont}
\SetKwInOut{Input}{input}
\SetKwInOut{Output}{output}

 \usepackage{cleveref} 
 \crefname{section}{\S\!}{\S\S\!}
 \crefname{theorem}{Theorem}{Theorems}
 \crefname{lemma}{Lemma}{Lemmas}
 \crefname{equation}{Equation}{Equations}
 \crefname{proposition}{Proposition}{Propositions}
 \crefname{claim}{Claim}{Claims}
\crefname{appendix}{\S\!}{\S\S\!}
   \crefname{algorithm}{Alg.}{Alg.}
 \crefname{figure}{Figure}{Figures}
 \crefname{table}{Table}{Tables}
 \crefname{remark}{Remark}{Remarks}
 \crefname{definition}{Definition}{Definitions}
 \crefname{equatinon}{Equation}{Equations}
 \crefname{corollary}{Corollary}{Corollaries}

\usepackage{thmtools}
\usepackage{thm-restate}

\let \oldtextcircled \textcircled
\renewcommand{\textcircled}[1]{\oldtextcircled{\footnotesize #1}}

\usepackage{enumitem}
\setlist[itemize]{leftmargin=9mm}

\usepackage[textsize=small]{todonotes}

\newcommand{\appendixtitle}[1]{
	\begin{center}
		\LARGE \bf #1
	\end{center}
}

\def\E{{\mathbb E}}
\def\X{{\cal X}}

\def\H{{\mathbb  H}}

\def\opt{\ensuremath{\Omega^*}}
\def\optcont{\ensuremath{\x^*}}

\def \c{\mathbf{c}}
\def \v{\mathbf{v}}

\def \r{\mathbf{r}}
\def \a{\mathbf{a}}
\def \b{\mathbf{b}}
\def \d{\mathbf{d}}
\def \x{\mathbf{x}}
\def \y{\mathbf{y}}
\def \s{\mathbf{s}}

\def \u{\mathbf{u}}

\def \bu{\mathbf{u}}

\def \BA{\mathbf{A}}

\def \BW{\mathbf{W}}

\def \R{{\mathbb{R}}}
\def \trans{\top}

\newcommand{\pare}[1]{{#1}}  

\newcommand{\maxcut}{{\textsc{MaxCut}}\xspace}

\newcommand{\argmax}{{\arg\max}}
\newcommand{\algname}[1]{{\texttt{#1}}}

\newcommand{\bas}{\mathbf{e}} 

\newcommand{\groundset}{\ensuremath{\mathcal{V}}}

\newcommand{\sete}[3]{\mathbf #1|_{#2} #3}

\usepackage{amsthm}

\newtheorem{theorem}{Theorem}
\newtheorem{lemma}{Lemma}

\newtheorem{observation}{Observation}

\newcommand{\parti}{\text{Z}} 
\newcommand{\zero}{\mathbf{0}} 
\newcommand{\one}{\mathbf{1}} 
\newcommand{\ele}{v} 
\newcommand{\multi}{f_{\text{mt}}} 
\newcommand{\bigo}[1]{\mathcal {O}\! \left(#1\right)}

\newcommand{\pa}{\text{PA}}
\newcommand{\intermed}{\mathbf{o}}
\def \bu{\mathbf{u}}

\def \m{\mathbf{m}}

\def \data {\text{D} }

\def \nmf {\text{CoordinateAscent}\xspace}
\def \dgmf {\text{DG-MeanField}\xspace}
\def \palong {\text{Posterior-Agreement}\xspace} 

\renewcommand{\mid}{|}
\newcommand{\kl}[2]{\mathbb{KL}(#1\|#2)}
\newcommand{\entropy}[1]{\mathbb{H}(#1)}

\allowdisplaybreaks

\usepackage[utf8]{inputenc} 

\usepackage{url}            
\usepackage{booktabs}   

\usepackage{amsfonts}       
\usepackage{nicefrac}       
\usepackage{microtype}      
\usepackage{wrapfig}

\usepackage{thm-restate}

\title{
	Optimal DR-Submodular Maximization and Applications to Provable Mean Field Inference
}

\author{
An Bian\\
	ETH Z{u}rich \\
	\texttt{ybian@inf.ethz.ch}\\
	 \And
	Joachim M. Buhmann\\
	ETH Z{u}rich\\
	\texttt{jbuhmann@inf.ethz.ch} \\
	 \And
	Andreas Krause\\
	ETH Z{u}rich\\
	\texttt{krausea@ethz.ch} \\
	\AND
	\normalfont{\large May 19, 2018
	\thanks{First appeared on arXiv on this date.}
	}
}

\begin{document}

\renewcommand{\b}{\mathbf{b}}
\renewcommand{\v}{\mathbf{v}}
\renewcommand{\u}{\mathbf{u}}
\renewcommand{\r}{\mathbf{r}}
\renewcommand{\d}{\mathbf{d}}
\renewcommand{\c}{\mathbf{c}}
\renewcommand{\O}{\mathcal{O}}

%
%
%

\maketitle


\begin{abstract}%

  Mean field inference 
  in probabilistic  models
   is generally a
  highly nonconvex problem.  
Existing optimization methods, e.g., coordinate ascent algorithms, can
only generate local optima.

 In this work we propose provable mean field methods for probabilistic
 log-submodular models and its posterior agreement (PA) with strong
 approximation guarantees.
 The main algorithmic technique is a new Double Greedy scheme, termed
 \algname{DR-DoubleGreedy}, for \textit{continuous} DR-submodular
 maximization with box-constraints.  This one-pass algorithm achieves the \textit{optimal} 1/2
 approximation ratio, which may be of independent interest. 
 We validate the superior performance of our algorithms with
 baseline results on both synthetic and real-world datasets.

\end{abstract}

%


\section{Introduction}

Consider the following scenario: You want to build a recommender
system for $n$ products to sell. Let $\groundset$ contain all the
products. The system is expected to recommend a subset of products
$S\subseteq \groundset$ to the user.  This recommendation 
should reflect relevance and diversity of the user's choice, such that
it will raise the readiness to buy.  The two most important components
in building such a system are (1) learning a utility function $F(S)$,
which measures the utility of any subset of products, and (2) inference,
i.e., finding the subset $\opt$ with the highest utility given the
learnt function $F(S)$.
The above task can be achieved using a class of probabilistic
graphical models that devise a distribution on all subsets of
$\groundset$. Such a distribution is known as a point process. 
Specifically, it defines $p(S)\propto \exp(F(S))$, which renders subset
of products $S$ with high utility to be very likely suggested.  In
general, inference in point processes is \#P-hard.  One resorts to
approximate inference methods via either variational techniques
\citep{wainwright2008graphical} or sampling.

In this paper we develop  mean field methods with provable guarantees.
Both of the two components in the  recommender system example above can
be achieved via provable mean field methods since (i) the latter  provide
approximate inference given a utility function and, (ii) by using
proper differentiation techniques, the iterative process of mean field
approximation can be unrolled to serve as a differentiable layer
\citep{zheng2015conditional}, thus enabling the backpropagation of the
training error to parameters of $F(S)$.  Thereby, learning $F(S)$ in
an end-to-end fashion can utilize modern deep learning and stochastic
optimization techniques.

The most important property which we require on $F(S)$ is
\emph{submodularity},
which
naturally models relevance and diversity. Djolonga et al. 
\cite{djolonga14variational} have used submodular functions $F(S)$ to
define two classes of point processes: 
$p(S)\propto \exp(F(S))$ is
termed probabilistic log-submodular models, while
$p(S)\propto \exp(-F(S))$ is called  probabilistic log-supermodular models.  
They
are strict generalizations of classical point processes, such as DPPs
\citep{kulesza2012determinantal}.
The variational techniques from
\citet{djolonga14variational,djolonga16mixed} focus on giving
tractable upper bounds of the log-partition functions.  This work
provides provable \emph{lower} bounds through mean field approximation, which
also completes the picture of variational inference for probabilistic
submodular models (PSMs).

The most frequently used algorithm for mean field approximation is
the \algname{\nmf} algorithm\footnote{It is known under various names in
  the literature, e.g., iterated conditional modes (ICM), naive mean
  field algorithm, etc.}. It maximizes the ELBO objective in a
coordinate-wise manner.
\algname{\nmf} has been shown to reach stationary points/local
optima. However, local optima may be arbitrarily poor, as we
demonstrate in \cref{sec_bad_localoptima}, and \algname{\nmf} would
get stuck in these poor local optima without extra techniques, which
motivates our pursuit to develop provable methods.

We firstly investigate the properties of mean field approximation for
probabilistic log-submodular models, and show that it falls into a
general class of nonconvex problem, called continuous DR-submodular
maximization with box-constraints. 
Continuous submodular optimization
is a class of well-behaved nonconvex programs, which has attracted 
increasingly more attention recently. 
Then we propose a new one-epoch algorithm for this general class of
nonconvex problem, called \algname{DR-DoubleGreedy}.  It achieves the
\textit{optimal} $1/2$ approximation ratio in linear time.
Lastly, we extend one-epoch algorithms to multiple epochs, resulting
in provable mean field algorithms, termed \algname{\dgmf}.

\textbf{Typical Application Domains.}
Recommender systems are just one illustrating example. There are
numerous scenarios that can benefit from the mean field method in this
work. These settings include, but not limited to, existing applications of
submodular models, such as diversity models
\citep{Tschiatschek16diversity,djolonga16mixed}, experimental design
using approximate submodular objectives
\citep{bianicml2017guarantees}, variable selection
\citep{Krause05nearoptimalnonmyopic}, data summarization
\citep{lin2011class},
dictionary learning \citep{krause2010submodular} etc.
Another category of applications is conducting model validation using
information-theoretic criteria. In order to infer the hyperparamters
in the model  $F(S)$, practitioners do validation by
splitting the training data into multiple folds, and then train models
on them.  Posterior-Agreement (PA,
\citep{Buhmann10isit,bian2016information}) provides an
information-theoretic criterion for the models trained on these
folds, to measure the fitness of one specific hyperparameter
configuration.  We show in \cref{sec_pa} that PA can be efficiently
approximated by the techniques developed in this work.

\paragraph{Contributions.} Motivated by the broad applicability
of mean field approximation, we contribute in the following respects:
i) We propose  the first \emph{optimal}
algorithm for the general problem of  continuous DR-submodular maximization with
box-constraints, which runs in linear time. 
ii) Based on the optimal algorithm, we propose provable mean field approaches for probabilistic
log-submodular models and its PA.
 iii) We also present
efficient polynomial methods to evaluate the multilinear extensions
for a large category of practical objectives, which
are used for optimizing the mean field objectives.

\subsection{Problem Statement and Related Work}

\textbf{Notation.}  Boldface letters, e.g. $\x$, represent vectors.
Boldface capital letters, e.g. $\BA$, denote  matrices. $x_i$ is the
$i^{\text{th}}$ entry of $\x$, $A_{ij}$ the $(ij)^{\text{th}}$ entry
of $\BA$.  We use $\bas_i$ to denote the standard $i^\text{th}$ basis
vector.  $f(\cdot)$ is used to specify a continuous function, and
$F(\cdot)$ to represent a set function.  $[n]:= \{1,...,n\}$.
Given two vectors $\x,\y$, $\x\leq \y$ means
$ \forall i, \,x_i\leq y_i$.  $\x\vee \y$ and $\x \wedge \y $ is
defined as coordinate-wise maximum and coordinate-wise minimum,
respectively. Finally, $\sete{x}{i}{k}$ is the operation of setting
the $i^\text{th}$ entry of $\x$ to $k$, while keeping all the others
unchanged, i.e., $\sete{x}{i}{k} = \x - x_i \bas_i + k\bas_i$.

All of the mean field approximation problems investigated in this work
fall into the following nonconvex maximization problem:
\begin{align}\label{opt_problem}
\underset{{ \x \in [\a, \;\b]}}{{\text {maximize}}} \;\; f(\x),    \tag{{\color{link_color}P}}
\end{align}
where $f: \X \rightarrow \R$ is continuous DR-submodular,
$\X = \prod_{i=1}^{n}\X_i$, each $\X_i$ is an interval
\citep{bach2015submodular,bian2017guaranteed}.  Continuous
DR-submodular functions define a subclass of continuous submodular
functions with the additional diminishing returns
(DR) property: $\forall \a\leq \b \in \X$,
$\forall i \in [n], \forall k\in \R_+$
it holds
$f(k\bas_i+\a) - f(\a) \geq f(k\bas_i+\b) - f(\b)$.
If $f$ is differentiable, DR-submodularity is equivalent to $\nabla f$
being an \textit{antitone} mapping from $\R^n$ to $\R^n$. If $f$ is
twice-differentiable, DR-submodularity is equivalent to all of the
entries of $\nabla^2 f(\x)$ being non-positive.
A function $f$ is DR-supermodular iff $-f$ is DR-submodular.

\textbf{Background \& Related Work. 
}
\label{cont_subopt_bk}
Submodularity is one of the most important properties in combinatorial
optimization and many applications for machine learning, with strong
implications for both guaranteed minimization and approximate
maximization in polynomial time \citep{krause2012submodular}.  
Continuous extensions of submodular set functions play an 
important role in submodular optimization, representative 
instances include Lov{\'a}sz extension \citep{lovasz1983submodular}, multilinear extension \citep{calinescu2007maximizing,DBLP:conf/stoc/Vondrak08,chekuri2014submodular,chekuri2015multiplicative}
and the softmax extension for DPPs \citep{gillenwater2012near}.
These
guaranteed optimizations have been advanced to continuous domains
recently, for both minimization
\citep{bach2015submodular,staib2017robust} and maximization
\citep{bian2017guaranteed,biannips2017nonmonotone,Wilder2017RiskSensitiveSO,chen2018online,mokhtari2018decentralized}.
Specifically, Bach \cite{bach2015submodular} studies continuous
submodular minimization without constraints. He also discusses the
possibility of using the technique for mean field inference of
probabilistic log-supermodular models.
\cite{bian2017guaranteed,biannips2017nonmonotone} characterize
continuous submodularity using the DR property and propose provable
algorithms for maximization.

Most related to this work is the classical problem of {unconstrained
  submodular maximization (USMs)}, which has been studied in binary
\citep{buchbinder2012tight}, integer \citep{soma2017non} and
continuous domains \citep{bian2017guaranteed}.
For the general problem \labelcref{opt_problem}, at first glance one
may consider discretization-based methods: Discretizing the continous
domain and transform \labelcref{opt_problem} to be an integer
optimization problem, then solve it using the reduction
\citep{ene2016reduction} or the integer Double Greedy algorithm
\citep{soma2017non}.  However, discretization-based methods are not
practical for \labelcref{opt_problem}: Firstly discretization will
inevitably introduce errors for the original continuous
problem \labelcref{opt_problem}; Secondly, the computational cost is
too high\footnote{e.g., the method from \cite{soma2017non} reaches
	$\frac{1}{2+\epsilon}$-approximation in
	$O( \frac{|\groundset|}{\epsilon}) \log( \frac{\Delta}{\delta})
	\log(B)(\theta + \log (B) )$
	time, $B$: \#grids of discretization, $\Delta$: the maximal
  positive marginal gain, $\delta$: minimum positive marginal gain}.
Thus we turn to continuous methods.
The shrunken Frank-Wolfe in \cite{biannips2017nonmonotone} provides
$1/e$ approximation guarantee and sublinear rate of convergence for
\labelcref{opt_problem}, but it is still computationally too
expensive:
In each iteration it has to calculate the full gradient, which costs
$n$ times as much as computing a partial derivative.

Based on the above analysis, the most promising algorithm to consider
would be the Double Greedy algorithm \citep{bian2017guaranteed}, which
needs to solve  $\bigo{n}$ 1-D subproblems, and achieves $1/3$
guarantee for continuous \emph{submodular} maximization. Since it only
needs $f(\x)$ to be continuous submodular, we call it
\algname{Submodular-DoubleGreedy} in the sequel.
In this work we propose a new 
Double Greedy scheme, achieving the 
optimal $1/2$ approximation ratio of \labelcref{opt_problem}.

%
%
%

%

%
%
%
%
%
 %
%
{\palong} (PA) is developed as an information-theoretic criterion for
model selection \citep{gorbach2017model} and algorithmic validation
\citep{informativemst,bian2016information}.  It originates from the
approximation set coding framework proposed by \cite{Buhmann10isit}.
Recently, \cite{buhmannaposterior} prove rigorous asymptotics of PA on
two typical combinatorial problems: Sparse minimum bisection and
Lawler's quadratic assignment problem.
\cite{djolonga14variational,djolonga15scalable} study variational
inference for PSMs, they propose L-Field to give upper bounds for
log-supermodular models through optimizing the
subdifferentials. However, they did not give tractable lower bounds
for probabilistic log-submodular models.

Along with the development of this work\footnote{\cite{niazadeh2018optimal} is a  contemporary  work, both papers were released on arXiv.}, 
\cite{niazadeh2018optimal} proposed an optimal algorithm 
for DR-submodular maximization. Their algorithm (Algorithm 4 in \cite{niazadeh2018optimal}, termed \algname{BSCB}: Binary-Search Continuous Bi-greedy) needs to estimate 
the partial derivative of the objective, which is not needed in our algorithm. 
Furthermore, our algorithm is arguably  
easier to interpret and implement than \algname{BSCB}.
We did extensive experiments (see \cref{sec_exp} for details on experimental statistics) to compare them, the results show 
that both algorithms generate promising solutions, however, 
 our algorithm produces better solutions than \algname{BSCB} in most of the experiments.

\section{Applications to Mean Field Approximation}

Mean field inference aims to approximate the intractable distribution
$p(S)\propto \exp(F(S))$ by a fully factorized surrogate distribution
$q(S|{\x}):= \prod_{i\in S}x_i \prod_{j\notin S}(1-x_j), \x\in[0,1]^n$.
This can be achieved by maximizing the \text{(ELBO)} objective, which
provides a lower bound for the log-partition function,
$ \text{(ELBO)}\leq \log\parti = \log \sum_{S\subseteq
	\groundset} \exp(F(S))$.  Specifically, the optimization problem is,
\begin{align}\notag 
  \max_{ \x \in [\zero, \one]} f(\x) 
  &  := 
    \overbrace{\E_{q(S\mid \x)}[F(S)]}^{\text{multilinear extension of }  F(S): \multi(\x)}\\\notag
 & - \sum\nolimits_{i=1}^{n} [x_i\log x_i + (1-x_i)\log(1-x_i)]  \\\label{opt_problem_meanfield}
  & \; = \multi(\x)   + \sum\nolimits_{i \in\groundset} H(x_i),
    \quad \text{(\textcolor{link_color}{ELBO})}
\end{align}
where $H(x_i) := -[ x_i \log x_i + (1-x_i)\log (1-x_i ) ] $ is the
binary entropy function and by default $0\log 0 =0$.
$\multi(\x) :=\E_{q(S\mid \x)}[F(S)]$ is the multilinear extension
\citep{DBLP:dblp_conf/ipco/CalinescuCPV07} of $F(S)$.  The above
(ELBO) is continuous DR-submodular w.r.t. $\x$, thus falling into the
general problem class \labelcref{opt_problem}.  At first glance,
$\multi(\x)$ seems to require an exponential number of operations for
evaluation; we show in \cref{sec_computing_multilinear} that
$\multi(\x)$ and its gradients can be computed precisely in polynomial
time for many classes of practical objectives, such as facility
location, FLID \citep{Tschiatschek16diversity}, set cover
\citep{lin2011optimal} and graph cuts.
Maximizing \text{(ELBO)} to optimality provides  the tightest
lower bound of $\log \parti$ in terms of the KL divergence
$\kl{q}{p}$.  We put details in \cref{sec_meanfield_lowerbounds}.

In addition to the traditional mean field objective (ELBO)
in \labelcref{opt_problem_meanfield}, here we further formulate a
second class of mean field objectives. They come from \palong (PA) for
probabilistic log-submodular models, which is an information-theoretic
criterion to conduct model and algorithmic validation
\citep{Buhmann10isit,buhmannaposterior,bian2016information}.

\subsection{Mean Field Inference  of \palong (PA)  
}\label{sec_pa}

Let us again consider the recommender example: usually there are some
hyperparameters in the model/utility function $F(S)$ that require
adaptation to the input data.
One natural way to do so is through model validation: Split the
training data into multiple folds, train a model on each fold $\data$
one would infer a ``noisy" posterior distribution $p(S\mid \data)$.  PA
measures the agreement between these ``noisy" posterior distributions.

Assume w.l.o.g. that  there are
two folds of data $\data', \data''$  in the sequel. 
In the PA framework, we have two consecutive 
targets:  1) Direct inference based on the two posterior  distributions 
$p(S\mid \data')$ and $p(S\mid \data'')$. 
This task amounts to find the MAP solution of the PA distribution
(which is discussed in the next paragraph),
it can be approximated by standard mean field inference.
2) Use the PA objective \labelcref{pa_objective} as a criterion for
model validation/selection. Since in general the PA
objective \labelcref{pa_objective} is intractable, we will still use
mean field lower bounds and some upper bounds in
\cite{djolonga14variational} to provide estimations for it.

\textbf{Mean Field Approximation of the Posterior-Agreement Distribution.}
\label{sec_pa_distribution}
A probabilistic log-submodular model is a special case of a Gibbs
random field with unit temperature and $-F(S)$ as the energy
function. In PA framework, we explicitly keep $\beta$ as the inverse
temperature,
$p_{\beta}(S | \data) := \frac{\exp(\beta F(S |
  \data))}{\sum_{\tilde{S} \subseteq \groundset}\exp(\beta F(\tilde S
  | \data))}, \forall S \subseteq \groundset$,
where $\data$ is the dataset used to train the model $F(S\mid \data)$.
The PA distribution is defined as,
\begin{align}\notag 
  p^{\pa}(S)  \propto p_{\beta}(S \mid \data') p_{\beta}(S\mid
  \data'') \propto \exp[ \beta ( F(S| \data') + F(S| \data'')  )]. 
\end{align}
Note that its log partition function is still intractable. In order to
approximate $p^{\pa}(S)$, we use mean field approximation with a
surrogate distribution
$q(S|{\x}):= \prod_{i\in S}x_i \prod_{j\notin S}(1-x_j)$,
 \begin{align}\notag 
   & \log \parti^{\pa} =  \log  \sum\nolimits_{S \subseteq \groundset}
     \exp[ \beta ( F(S| \data') + F(S| \data'')  )] \\\label{pa_elbo} 
   & \geq \beta \ \mathbb{E}_{q(S|{\x})} [F(S|\data')] +\beta \
     \mathbb{E}_{q(S|{\x})} [F(S|\data'')] \\\notag 
     &   \quad   + \sum\nolimits_{i
     \in\groundset} H(x_i).    \quad\quad\qquad\qquad\qquad
     \text{(\textcolor{link_color}{PA-ELBO})} 
 \end{align}
 Maximizing \text{(PA-ELBO)} in \labelcref{pa_elbo} still falls into
 the general problem class \labelcref{opt_problem} (see
 \cref{sec_meanfield_lowerbounds} for details).  Maximizing
 \text{(PA-ELBO)} also serves as a building block for the second
 target below.

\textbf{Lower Bounds for the  \palong\ Objective.}
The \pa\ objective is used to measure the agreement between the two
posterior distributions motivated by an information-theoretic analogy
\citep{buhmannaposterior,bian2016information}. By introducing the same
surrogate distribution $q(S\mid\x)$, one can easily derive that,
\begin{align}\label{pa_objective}
&  \mathrm{log} \sum\nolimits_{S \subseteq \groundset}
  p_{\beta}(S \mid \data') p_{\beta}(S \mid \data'')   \quad
  \text{(PA\  objective)} \\\notag  
&  \!\!\!\!\!\!\geq  \underbrace{\H(q) \!+ \!\beta \ {\E}_q
  F(S|\data') \!+\!\beta \ {\E}_q F(S|\data'')}_{\text{(PA-ELBO)
  in \labelcref{pa_elbo}}}  \\\notag 
&    -\log \parti(\beta; \data') - \log \parti(\beta; \data'') 
\end{align}
where $\H(q)$ is the entropy of $q$,  $\parti(\beta; \data')$ and $\parti(\beta; \data'')$ are the
partition functions of the two noisy distributions, respectively.
In order to find the best lower bound for \pa, one need to maximize
w.r.t. $q(S|\x)$ the (PA-ELBO) objective, at the same time, find the
upper bounds for
$\log \parti(\beta; \data') + \log \parti(\beta; \data'')$.  The
latter can be achieved using techniques from
\citep{djolonga14variational}.  We summarize the details in
\cref{upperbounds_pa} to make it self-contained.

\section{An Optimal Algorithm for Continuous DR-Submodular
  Maximization}

\label{sec_dr_doublegreedy}

Unfortunately, problem \labelcref{opt_problem} is generally hard: 
The $1/2$ hardness result \citep[Proposition 5]{bian2017guaranteed}
can be easily translated to \labelcref{opt_problem} with details
deferred to \cref{supp_hardness}.  The following question arises
naturally: Is it possible to achieve the optimal $1/2$ approximation
ratio (unless RP=NP)
by properly utilizing the extra DR propety in \labelcref{opt_problem}?
To affirmatively answer this question, we propose a new Double Greedy
algorithm for continuous DR-submodular maximization called
\algname{\algname{DR-DoubleGreedy}} and prove a $1/2$ approximation
ratio.

\subsection{A Deterministic $1/2$-Approximation for Continuous DR-Submodular Maximization}

\vspace{-.2cm}
\begin{algorithm}[ht]
	\caption{\algname{DR-DoubleGreedy}$(f, \a, \b)$
        }\label{alg_cont_doublegreedy} 
	\KwIn{ $\max_{\x \in [\a, \b]}f(\x)$,  $f(\x)$ is
          {\color{blue}DR}-submodular,  $[\a,\b]\subseteq  \X$ 
	}
	{$\x^0 \leftarrow \a$,
		$\y^0 \leftarrow \b$\;}
	\For{$k = 1 \rightarrow n$}{
		{ let $v_k$  be the coordinate being operated\;}
		{find
			$ u_a$ \text{ such that }
			$f(\sete{x^{k-1}}{\ele_k}{u_a}) \geq \max_{u'}
                        f(\sete{x^{k-1}}{\ele_k}{ u'}) -
                        \frac{\delta}{n}$, 

			$\delta_a \leftarrow f (\sete{x^{k-1}}{\ele_k}{u_a}) -
			f(\x^{k-1}$)\label{1d_1} \;}

		{find $u_b$ \text{ such that }
			$f(\sete{y^{k-1}}{\ele_k}{u_b})\geq \max_{u'}
                        f(\sete{y^{k-1}}{\ele_k}{u'}) -
                        \frac{\delta}{n}$, 

			$\delta_b \leftarrow f (\sete{y^{k-1}}{\ele_k}{u_b}) -
			f(\y^{k-1})$\label{1d_2} \;}

		{$\x^k \leftarrow \sete{x^{k-1}}{\ele_k}{(
                    \frac{\delta_a}{\delta_a + \delta_b}   u_a +
                    \frac{\delta_b}{\delta_a + \delta_b}
                    u_b)}$\tcp*{update $\ele_k^\text{th}$ coordinate  to be a
                    \emph{\color{blue}convex} combination of $u_a$ \& 
                    $u_b$}} 

		{$\y^k \leftarrow \sete{y^{k-1}}{\ele_k}{(
                    \frac{\delta_a}{\delta_a + \delta_b}   u_a +
                    \frac{\delta_b}{\delta_a + \delta_b}   u_b)}$\;} 

	}
	\KwOut{$\x^n$ or $\y^n \; (\x^n = \y^n)$ }
\end{algorithm}
The pseudocode of \algname{DR-DoubleGreedy} as summarized in
\cref{alg_cont_doublegreedy} describes a one-epoch algorithm, sweeping
over the $n$ coordinates in one pass.
Like the previous Double Greedy algorithms, the procedure maintains
two solutions $\x, \y$, that are initialized as the lower bound $\a$
and the  upper bound $\b$, respectively.  In iteration $k$, it operates on
coordinate $\ele_k$, and solves the two 1-D subproblems
$\max_{u'} f(\sete{x^{k-1}}{\ele_k}{ u'})$ and
$\max_{u'} f(\sete{y^{k-1}}{\ele_k}{u'})$, based on $\x^\pare{k-1}$
and $ \y^\pare{k-1}$, respectively. It also allows   solving 1-D subproblems approximately with 
  additive error $\delta \geq 0$ ($\delta =0$ recovers the error-free case).
Let $u_a$ and $u_b$ be the solutions of these 1-D subproblems.

Unlike previous Double Greedy algorithms, we change coordinate
$\ele_k$ of $\x^\pare{k-1}$ and $ \y^\pare{k-1}$ to be a \emph{convex}
combination of $u_a$ and $u_b$, weighted by respective gains
$\delta_a$, $\delta_b$. This convex combination is the key step that
utilizes the DR property of $f$, and it also plays a crucial role in
the proof.

Note that the 1-D subproblem has a closed-form solution for
ELBO \labelcref{opt_problem_meanfield} (and similarly for
PA-ELBO \labelcref{pa_elbo}).  For coordinate $i$, the partial
derivative of the multilinear extension is $\nabla_i\multi(\x)$, and for
the entropy term, it is $\nabla H(x_i) = \log \frac{1-x_i}{x_i}$. Then $x_i$
should be updated as
$x_i \leftarrow \sigma(\nabla_i \multi(\x)) = \bigl(1+ \exp(-
    \nabla_i \multi(\x)\bigr)^{-1}$,
where $\sigma$ is the logistic sigmoid function.
\begin{restatable}[]{theorem}{restatheoremone}
\label{thm_cont_doublegreedy}
	Assume the optimal solution of  $\max_{\x \in [\a, \b]}f(\x)$  is $\optcont$, then for
	\cref{alg_cont_doublegreedy}    it holds,
	\begin{flalign}
	f(\x^n) \geq \frac{1}{2}f(\optcont) + \frac{1}{4}[f(\a) + f(\b)] - \frac{5\delta}{4}.
	\end{flalign}
\end{restatable}
\textbf{Proof Sketch.}  The high level 
proof strategy is to bound the change of an intermediate variable
$\intermed^\pare{k} := (\optcont \vee \x^k) \wedge \y^k$ through the
course of \cref{alg_cont_doublegreedy}, which is the common framework
in the analysis of all existing Double Greedy variants
\cite{buchbinder2012tight,gottschalk2015submodular,bian2017guaranteed,soma2017non}%
\footnote{Note
  that \cite{buchbinder2012tight} analyzed in the  appendix a Double
  Greedy variant (Alg. 4 therein) for maximizing the multilinear
  extension of a submodular \textit{set} function, which is a special
  case of continuous DR-submodular functions.
 However, that variant 
 cannot be applied for the general DR-submodular 
 objective in \labelcref{opt_problem}; Furthermore,  the analysis 
 for that variant  is not applicable nor generalizable   for  \labelcref{opt_problem}, since it only shows the guarantee wrt. the optimal solution 
 that must be {binary}. While 
 the optimal solution to \labelcref{opt_problem} could be any fractional point
 in $[\a,\b]$.}. 
The novelty of our method results from the update of $\x$, $\y$, which
plays a key role in achieving the optimal $1/2$ approximation ratio.
Furthermore, in the analysis we find a way to utilize the DR property
directly, resulting in a succinct proof.  We document the
details in \cref{supp_proof_dr}, and summarize a sketch here.
Firstly, using DR-submodularity,  we prove that in each iteration, if we were to
flip the 1-D subproblem solutions of $\x$ and $\y$,
it  still does not decrease the function value (in the error-free case $\delta=0$).
\begin{restatable}[]{lemma}{restalemmaone}
\label{lemma1}
		For all $k=1,...,n$, it holds that,
		\begin{align}
			f(\sete{x^\pare{k-1}}{\ele_k}{u_b}) - f(\x^\pare{k-1}) \geq -{\delta}/{n},  \\\notag 
			 f(\sete{y^\pare{k-1}}{\ele_k}{u_a}) - f(\y^\pare{k-1} )  \geq -{\delta}/{n}.
		\end{align}
\end{restatable}
Then using the new update rule and the DR property, we show that the
loss on intermediate variables
$f( \intermed^\pare{k-1} ) - f(\intermed^\pare{k})$ can be upper
bounded by the increase of the objective value in $\x$ and $\y$ times
$1/2$.

\begin{restatable}[]{lemma}{restalemmatwo}
\label{lemma2}
For all $k=1,...,n$, it holds that,
\begin{align}
  &f( \intermed^\pare{k-1} )	 - f(\intermed^\pare{k}) \\\notag 
  	 \leq
  & \frac{1}{2}   \left[ f(\x^\pare{k}) -f(\x^\pare{k-1} ) + f(\y
  ^\pare{k} ) - f(\y ^\pare{k-1})   \right]  + \frac{2.5 \delta }{n}. 
\end{align}
\end{restatable}
Given  \cref{lemma2}, let
	us sum for $k = 1,...,n$.
	After   rearrangement it reaches the final conclusion.

\subsection{Multi-epoch Extensions
}
Though \algname{DR-DoubleGreedy} reaches the optimal
$1/2$ guarantee with  one epoch, in practice
it usually helps to use its output as an initializer,
and continue optimizing coordinate-wisely  for
additional epochs. Since each step of  coordinate update
will never decrease the function value, the
approximation guarantees will hold.
We call this class of algorithms \algname{DoubleGreedy-MeanField},
abbreviated as \algname{DG-MeanField}, and summarize the pseudocode in
\cref{alg_dg_meanfield}.
\begin{algorithm}[htbp]
	\caption{\algname{DG-MeanField-$1/2$} \& \algname{DG-MeanField-$1/3$} }\label{alg_dg_meanfield}
	\KwIn{ $\max_{\x \in [\a, \b]}f(\x)$,  e.g., from
	the ELBO \labelcref{opt_problem_meanfield} or  PA-ELBO \labelcref{pa_elbo} objective
	}
	{Option I: \algname{DG-MeanField-$1/3$}: run \algname{Submodular-DoubleGreedy} \citep{bian2017guaranteed} to get a ${1}/{3}$ initializer $\hat\x$}
	
	{Option II: \algname{DG-MeanField-$1/2$}: run \algname{DR-DoubleGreedy} to get a $1/2$ initializer $\hat\x$  \;}
	{ beginning with $\hat\x$, optimize $f(\x)$ coordinate by coordinate for $T$ epochs \;}
\end{algorithm}
\vspace{-.4cm}

\section{Efficient Methods  for Calculating Multilinear Extension \&  Gradients}
\label{sec_computing_multilinear}

In this section we present guaranteed methods to efficiently calculate
the multilinear extension $\multi (\x)$ and its gradients in
polynomial time\footnote{\cite{iyer2015submodular} give closed-form
  expressions for the partition functions of submodular point
  processes for several classes of objectives, which can be treated as
  the multilinear extensions evaluated at $0.5*\one$ with proper
  scaling.}.  Remember that the multilinear extension is the expected
value of $F(S)$ under the surrogate distribution:
$\multi(\x) := \E_{q(S\mid \x)} [F(S)] =\sum_{S\subseteq
  \groundset}F(S)\prod_{i\in S}x_i \prod_{j\notin S}(1-x_j).$
One can verify that the partial derivative of $\multi(\x)$ is,
\begin{align}\notag 
\nabla_i \multi(\x)  & 
  = \E_{q(S\mid \x, x_i = 1)} [F(S)]  - \E_{q(S\mid \x, x_i=0)} [F(S)]\\\notag 
  & =\multi(\sete{x}{i}{1}) - \multi(\sete{x}{i}{0})\\\notag  
                    &  =  \sum_{S\subseteq \groundset, S\ni i } F(S)
                       \prod_{j \in S\backslash\{i\}}x_j
                       \prod_{j'\notin S}(1-x_{j'})\\\notag 
                & \quad  - \sum_{S\subseteq
                       \groundset\backslash \{i\} }\ F(S) \prod_{j\in
                       S} x_j \prod_{j'\notin S, j'\neq i}
                       (1-x_{j'}). 
\end{align}

\subsection{Gibbs Random Fields with Finite Order of Interactions}\label{gibbs_multilinear}

Let us use $\v \in \{0,1\}^\groundset  $ to equivalently 
denote the $n$ binary random variables. 
$F(\v)$ corresponds to the negative energy function 
in Gibbs random fields. If the energy function is
parameterized with a finite order of interactions, i.e., 
$F(\v) = \sum_{s\in \groundset} \theta_s v_s + \sum_{(s,t)\in
  \groundset \times \groundset} \theta_{s, t}v_s v_t + ... +
\sum_{(s_1, s_2, ..., s_d)} \theta_{s_1, s_2, ..., s_d}v_{s_1} \cdots
v_{s_d},  \; d< \infty$, then one can verify that its
multilinear extension has the following closed form,
\begin{align}
\multi(\x)
= \sum_{s\in \groundset} \theta_s x_s + \sum_{(s,t)\in \groundset
  \times \groundset} \theta_{s, t}x_s x_t + ...  \\\notag 
+  \sum_{(s_1, s_2,
  ..., s_d)} \theta_{s_1, s_2, ..., s_d}x_{s_1} \cdots  x_{s_d}\,.
\end{align}
The gradient of this expression can also be easily derived.  Given
this observation, one can quickly  derive the multilinear extensions
of a large category of energy functions of Gibbs random fields, e.g.,
graph cut, hypergraph cut, Ising models, etc.  Details are in
\cref{supp_multilinear_closedform}.

\subsection{Facility Location \& FLID (Facility Location Diversity)}

FLID is a diversity model \citep{Tschiatschek16diversity} that has
been designed as a computationally efficient alternative to DPPs.  It
is in a more general form than facility location.  Let
$\BW\in \R_+^{|\groundset|\times D }$ be the weights, each row
correponds to the latent representation of an item, with  $D$ as the
dimensionality. Then
\begin{align}\notag
  F(S) := & \sum\nolimits_ {i\in S} u_i + \sum\nolimits_{d=1}^{D} (
            \max_{i\in S} W_{i,d} - \sum\nolimits_{i\in S} W_{i,d} ) \\ 
=& \sum\nolimits_ {i\in S} u'_i + \sum\nolimits_{d=1}^{D}
            \max_{i\in S}W_{i,d}, 
\end{align}
which models both coverage and diversity, and
$u'_i = u_i - \sum_{d=1}^{D} W_{i,d}$. If $u'_i=0$, one recovers the
facility location objective.
The computational complexity of evaluating its partition function is
$\bigo{|\groundset|^{D+1}}$ \citep{Tschiatschek16diversity}, which is
exponential in terms of $D$.

We now show the technique such that $\multi(\x)$ and
$\nabla_i \multi(\x) $ can be evaluated in $\bigo{Dn^2}$ time.
Firstly, for one $d\in [D]$, let us sort $W_{i,d}$ such that
$W_{i_d(1), d} \leq W_{i_d(2), d} \leq \cdots \leq W_{i_d(n), d} $.
After this sorting, there are $D$ permutations to record:
$i_d(l), l=1,...,n, \forall d\in [D]$.  Now, one can verify that,
\begin{align}\notag 
  & \multi(\x) \\\notag 
   & =  \sum_ {i \in [n]} u'_i x_i +  \sum_d
                \sum_{S\subseteq \groundset }  \max_{i \in S} W_{i, d}
                \prod_{m\in S}x_m \prod_{m'\notin S}(1-x_{m'}) \\\notag 
              & = \sum_ {i\in [n]} u'_i x_i + \sum_{d} \sum_{l=1}^n
                W_{i_d(l), d} x_{i_d(l)} \prod_{m=l+1}^n [1-
                x_{i_d(m)}].  
\end{align}
Sorting costs $\bigo{Dn\log n}$, and from the above expression, one
can see that the cost of evaluating $\multi(\x)$ is $\bigo{Dn^2}$. By the
relation that
$\nabla_i \multi(\x) = \multi(\sete{x}{i}{1}) -
\multi(\sete{x}{i}{0})$,
the cost is also $\bigo{Dn^2}$.  For $\nabla_i \multi(\x)$, there exists a
refined way to calculate this derivative, which we explain in
\cref{supp_multilinear_closedform}.

\subsection{Set Cover Functions}
\label{supp_setcover}
Suppose there are $|C| = \{c_1, ...,c_{|C|}\}$ concepts, and $n$ items
in $\groundset$. Give a set $S\subseteq \groundset$, $\Gamma (S)$
denotes the set of concepts covered by $S$. Given a modular function
$\m: 2^C \mapsto \R_+ $, the set cover function is defined as
$F(S) = \m (\Gamma(S))$.
This function models coverage in maximization,
and also the notion of complexity in minimization problems \citep{lin2011optimal}. 
Let us define an inverse map $\Gamma^{-1}$, such that for 
each concept $c$, $\Gamma^{-1}(c)$ denotes the set 
of items $v$ such that $\Gamma^{-1}(c) \ni v$. So the 
multilinear extension is,
\begin{align}\notag 
\multi(\x)  & =  \sum\nolimits_{i \in \groundset}  \m (\Gamma(S))  \prod\nolimits_{m\in S}x_m \prod\nolimits_{m'\notin S }(1-x_{m'}) \\
 & =  \sum\nolimits_ {c\in C}  m_c \left[  1- \prod\nolimits_{ i\in \Gamma^{-1}(c) }  (1- x_i) \right].
\end{align}
The last equality is achieved by considering the situations
where a concept $c$ is covered.
One can observe that both $\multi(\x)$ and $\nabla_i \multi(\x) $ can
be evaluated in $\bigo{n|C|}$ time.

\subsection{General Case: Approximation by Sampling}

In the most general case, one may only have access to the function values of $F(S)$. 
In this scenario, one can use a polynomial number of sample steps to estimate
$\multi(\x)$ and its gradients. Specifically: 1) Sample $k$ times
$S \sim q(S|\x)$ and evaluate function values for them, resulting in
$F(S_1), ...,F(S_k)$.  2) Return the average
$\frac{1}{k}\sum_{i=1}^{k} F(S_i)$. According to the Hoeffding bound
\citep{hoeffding1963probability}, one can easily derive that
$\frac{1}{k}\sum_{i=1}^{k} F(S_i)$ is arbitrarily close to
$\multi(\x)$ with increasingly more samples: With probability at least
$1- \exp(-k\epsilon^2/2)$, it holds that
$|\frac{1}{k}\sum_{i=1}^{k} F(S_i) - \multi(\x)| \leq \epsilon \max_S
|F(S)| $, for all $\epsilon > 0$.

\section{Experiments}
\label{sec_exp}

The objectives under investigation are
ELBO \labelcref{opt_problem_meanfield} and
PA-ELBO \labelcref{pa_elbo} (We set  $\beta = 1$ in PA-ELBO).  We tested on the representative FLID
model on the following algorithms and baselines: 

The first category is
one-epoch algorithms, including \textcircled{1}
\algname{Submodular-DoubleGreedy} from \cite{bian2017guaranteed} with
$1/3$ guarantee,  \textcircled{2} \algname{BSCB}  (Algorithm 4 in \cite{niazadeh2018optimal}, termed Binary-Search Continuous Bi-greedy,
where we chose $\epsilon=10^{-3}$) with $1/2$ guarantee 
and \textcircled{3} \algname{DR-DoubleGreedy}
(\cref{alg_cont_doublegreedy}) with $1/2$ guarantee.

The second category contain multiple-epoch algorithms: \textcircled{4}
 \algname{CoordinateAscent-0}: initialized as $\zero$ and coordinate-wisely improving  the solution;
\algname{CoordinateAscent-1}: initialized as $\one$;
\algname{CoordinateAscent-Random}: initialized as a uniform  vector $U(\zero, \one)$.
\textcircled{5} \algname{\dgmf-${1}/{3}$}.  \textcircled{6}
\algname{\dgmf-${1}/{2}$} from \cref{alg_dg_meanfield}.
\textcircled{7} \algname{BSCB-Multiepoch}, which is the
multi-epoch extension of \algname{BSCB}: After the first
epoch of running \algname{BSCB}, it continues to improve the solution coordinate-wisely. For all algorithms, we use the same random order to process
the coordinates within each  epoch.

{
	\begin{table}
		\begin{center}
		\footnotesize
			\caption{Summary of results on ELBO objective \labelcref{opt_problem_meanfield} and PA-ELBO objective \labelcref{pa_elbo}. \algname{Sub-DG} stands for  \algname{Submodular-DoubleGreedy}, \algname{DR-DG} stands for  \algname{DR-DoubleGreedy}. Boldface numbers indicate the best mean of function values returned. For ELBO,  the mean and standard deviation were calculated
				for 10 FLID models trained on 10 folds of the data. For PA-ELBO, the mean and standard deviation were calculated  for models trained over $45$ pairs of folds. More details are in the experimental section.}
			\label{tab_summ_elbo}
			\begin{tabular}{|c|c||c|c|c||c|c|c|}
				\cline{1-8} 
				&  &  \multicolumn{3}{|c||}{\textbf{ELBO} objective \labelcref{opt_problem_meanfield}} &  \multicolumn{3}{|c|}{\textbf{PA-ELBO} objective \labelcref{pa_elbo}}  \\
				\cline{1-8}
				Category	& $D$ &  \algname{Sub-DG} & \algname{BSCB} & \algname{DR-DG} &  \algname{Sub-DG} & \algname{BSCB} & \algname{DR-DG} \\
				\hline
				\hline
				\multirow{2}{*}{furniture}  & 2  & 2.078$\pm$0.091 & 2.771$\pm$0.123  & \textbf{3.035}$\pm$0.059  & 0.918$\pm$0.768 & 2.287$\pm$0.399  & \textbf{2.402}$\pm$0.159\\
				& 3  & 1.835$\pm$0.156 & 2.842$\pm$0.128  & \textbf{3.026}$\pm$0.099 & 1.296$\pm$1.176 & 2.536$\pm$0.439  & \textbf{2.693}$\pm$0.181 \\
				$n$=32 & 10  & 1.375$\pm$0.194 & \textbf{2.951}$\pm$0.161  & {2.917}$\pm$0.103 & 1.504$\pm$1.110 & 2.764$\pm$0.405  & \textbf{2.882}$\pm$0.248 \\
				\hline
				\multirow{2}{*}{carseats}  & 2  & 2.089$\pm$0.166 & 2.863$\pm$0.090  & \textbf{3.045}$\pm$0.069 & 1.015$\pm$1.081 & 2.106$\pm$0.228  & \textbf{2.348}$\pm$0.219\\
				& 3  & 1.890$\pm$0.146 & 3.003$\pm$0.110  & \textbf{3.138}$\pm$0.082 & 1.309$\pm$1.218 & 2.414$\pm$0.267  & \textbf{2.707}$\pm$0.208 \\
				$n$=34 & 10  & 1.390$\pm$0.232 & \textbf{3.100}$\pm$0.140  & {3.003}$\pm$0.157 & 1.599$\pm$1.317 & 2.684$\pm$0.271  & \textbf{2.915}$\pm$0.250 \\
				\hline
				\multirow{2}{*}{safety}  & 2  & 1.934$\pm$0.402 & 2.727$\pm$0.212  & \textbf{2.896}$\pm$0.098 & 1.370$\pm$1.203 & 2.049$\pm$0.280  & \textbf{2.341}$\pm$0.161\\
				& 3  & 1.867$\pm$0.453 & 2.830$\pm$0.191  & \textbf{2.970}$\pm$0.110 & 1.706$\pm$1.296 & 2.288$\pm$0.297  & \textbf{2.619}$\pm$0.167 \\
				$n$=36 & 10  & 1.546$\pm$0.606 & 2.916$\pm$0.191  & \textbf{2.920}$\pm$0.149 & 1.948$\pm$1.353 & 2.467$\pm$0.270  & \textbf{2.738}$\pm$0.187 \\
				\hline
				\multirow{2}{*}{strollers}  & 2  & 2.042$\pm$0.181 & 2.829$\pm$0.144  & \textbf{2.928}$\pm$0.060 & 0.865$\pm$0.952 & 1.933$\pm$0.256  & \textbf{2.202}$\pm$0.226\\
				& 3  & 1.814$\pm$0.264 & 2.958$\pm$0.146  & \textbf{2.978}$\pm$0.077 & 1.172$\pm$1.063 & 2.181$\pm$0.297  & \textbf{2.543}$\pm$0.254 \\
				$n$=40 & 10  & 1.328$\pm$0.544 & \textbf{3.065}$\pm$0.162  & 2.910$\pm$0.140 & 1.702$\pm$1.334 & 2.480$\pm$0.304  & \textbf{2.767}$\pm$0.336 \\
				\hline
				\multirow{2}{*}{media}  & 2  & 3.221$\pm$0.066 & 3.309$\pm$0.055  & \textbf{3.493}$\pm$0.051 & 0.372$\pm$0.286 &\textbf{ 1.477}$\pm$0.128  & 1.336$\pm$0.101\\
				& 3  & 3.276$\pm$0.082 & 3.492$\pm$0.083  & \textbf{3.712}$\pm$0.079 & 0.418$\pm$0.366 & 1.736$\pm$0.177  & \textbf{1.762}$\pm$0.095 \\
				$n$=58 & 10  & 2.840$\pm$0.183 & 3.894$\pm$0.122  & \textbf{3.924}$\pm$0.114 & 0.653$\pm$0.727 & 2.309$\pm$0.244  & \textbf{2.524}$\pm$0.130 \\
				\hline
				\multirow{2}{*}{health}  & 2  & 3.197$\pm$0.067 & 3.174$\pm$0.074  & \textbf{3.516}$\pm$0.043 & 0.548$\pm$0.282 & \textbf{1.655}$\pm$0.122  & 1.650$\pm$0.073\\
				& 3  & 3.231$\pm$0.055 & 3.306$\pm$0.108  & \textbf{3.707}$\pm$0.064 & 0.649$\pm$0.413 & 1.903$\pm$0.173  & \textbf{2.025}$\pm$0.083 \\
				$n$=62 & 10  & 2.633$\pm$0.115 & 3.508$\pm$0.120  & \textbf{3.675}$\pm$0.110 & 0.768$\pm$0.628 & 2.233$\pm$0.196  & \textbf{2.375}$\pm$0.101 \\
				\hline
				\multirow{2}{*}{toys}  & 2  & 3.543$\pm$0.047 & 3.454$\pm$0.091  & \textbf{3.856}$\pm$0.044 & 0.597$\pm$0.480 & 1.731$\pm$0.182  & \textbf{{1.761}}$\pm$0.133\\
				& 3  & 3.362$\pm$0.055 & 3.412$\pm$0.070  & \textbf{3.736}$\pm$0.051 & 0.578$\pm$0.520 & 1.738$\pm$0.192  & \textbf{1.802}$\pm$0.151 \\
				$n$=62 & 10  & 3.037$\pm$0.138 & 3.706$\pm$0.108  & \textbf{3.859}$\pm$0.119 & 0.758$\pm$0.871 & 2.140$\pm$0.242  & \textbf{2.330}$\pm$0.177 \\
				\hline
				\multirow{2}{*}{diaper}  & 2  & 3.500$\pm$0.058 & 3.517$\pm$0.058  & \textbf{3.636}$\pm$0.043 & 0.295$\pm$0.158 & \textbf{1.119}$\pm$0.063  & 0.665$\pm$0.116\\
				& 3  & 3.739$\pm$0.080 & 3.753$\pm$0.065  & \textbf{3.974}$\pm$0.065 & 0.337$\pm$0.240 & \textbf{1.429}$\pm$0.111  & 1.141$\pm$0.120 \\
				$n$=100 & 10  & 3.423$\pm$0.110 & 4.150$\pm$0.120  & \textbf{4.203}$\pm$0.086 & 0.386$\pm$0.504 & 1.969$\pm$0.201  & \textbf{2.009}$\pm$0.199 \\
				\hline
				\multirow{2}{*}{feeding}  & 2  & 3.942$\pm$0.041 & 3.808$\pm$0.024  & \textbf{3.970}$\pm$0.036 & 0.393$\pm$0.034 & \textbf{0.894}$\pm$0.022  & 0.501$\pm$0.029\\
				& 3  & 4.333$\pm$0.031 & 4.095$\pm$0.032  & \textbf{4.390}$\pm$0.031 & 0.503$\pm$0.072 & \textbf{1.232}$\pm$0.041  & 0.893$\pm$0.046 \\
				$n$=100 & 10  & 4.611$\pm$0.053 & 4.553$\pm$0.079  & \textbf{4.860}$\pm$0.056 & 0.608$\pm$0.239 & 1.808$\pm$0.087  & \textbf{1.820}$\pm$0.078 \\
				\hline
				\multirow{2}{*}{gear}  & 2  & 3.311$\pm$0.046 & 3.150$\pm$0.037  & \textbf{3.430}$\pm$0.040 & 0.232$\pm$0.068 & \textbf{1.019}$\pm$0.048  & 0.590$\pm$0.043\\
				& 3  & 3.538$\pm$0.048 & 3.347$\pm$0.045  & \textbf{3.721}$\pm$0.050 & 0.303$\pm$0.132 & \textbf{1.257}$\pm$0.085  & 1.020$\pm$0.064 \\
				$n$=100 & 10  & 3.065$\pm$0.083 & 3.550$\pm$0.050  & \textbf{3.670}$\pm$0.067 & 0.312$\pm$0.232 & \textbf{1.566}$\pm$0.130  & 1.514$\pm$0.072 \\
				\hline
				\multirow{2}{*}{bedding}  & 2  & 3.406$\pm$0.080 & 3.374$\pm$0.088  & \textbf{3.620}$\pm$0.062 & 0.525$\pm$0.121 & 1.932$\pm$0.194  & \textbf{2.001}$\pm$0.080\\
				& 3  & 3.648$\pm$0.106 & 3.564$\pm$0.083  & \textbf{3.876}$\pm$0.081 & 2.499$\pm$0.972 & 2.250$\pm$0.269  & \textbf{2.624}$\pm$0.066 \\
				$n$=100 & 10  & 3.355$\pm$0.161 & 3.799$\pm$0.144  & \textbf{3.912}$\pm$0.082 & \textbf{3.919}$\pm$0.045 & 2.578$\pm$0.358  & {3.157}$\pm$0.091 \\
				\hline
				\multirow{2}{*}{apparel}  & 2  & 3.560$\pm$0.094 & 3.527$\pm$0.046  & \textbf{3.784}$\pm$0.059 & 0.268$\pm$0.109 & \textbf{1.552}$\pm$0.141  & 1.513$\pm$0.191 \\
				& 3  & 3.878$\pm$0.092 & 3.755$\pm$0.062  & \textbf{4.140}$\pm$0.063 & 0.490$\pm$0.677 & 1.900$\pm$0.237  & \textbf{2.225}$\pm$0.136\\
				$n$=100 & 10  & 3.751$\pm$0.087 & 4.084$\pm$0.075  & \textbf{4.425}$\pm$0.066 & 0.820$\pm$1.372 & 2.351$\pm$0.337  & \textbf{2.967}$\pm$0.150 \\
				\hline
				\multirow{2}{*}{bath}  & 2  & 2.957$\pm$0.087 & 3.024$\pm$0.032  & \textbf{3.198}$\pm$0.056 & 0.197$\pm$0.090 & \textbf{1.101}$\pm$0.083  & 0.795$\pm$0.078\\
				& 3  & 3.062$\pm$0.085  & 3.195$\pm$0.058  & \textbf{3.448}$\pm$0.058 & 0.247$\pm$0.163 & \textbf{1.368}$\pm$0.134  & 1.269$\pm$0.059 \\
				$n$=100 & 10  & 2.497$\pm$0.135 & 3.426$\pm$0.076  & \textbf{3.438}$\pm$0.089 & 0.327$\pm$0.312 & 1.711$\pm$0.183  & \textbf{1.742}$\pm$0.098 \\
				\hline
			\end{tabular}
		\end{center}
	\end{table}
} %

We are trying to understand:
1)  In terms of continuous DR-submodular maximization,   how good are  the  solutions returned by one-epoch
algorithms?
2)
How good are the realized lower bounds?
  For small scale problems 
  we can calculate the true
  log-partitions exhaustively, which servers as
  a natural upper bound of ELBO.
All algorithms and subroutines are implemented in Python3, and source
code will be released soon.

\setkeys{Gin}{width=0.9\textwidth}
\begin{figure}[htb]
	\small
	\center 
	\vspace{-.2cm}
	\includegraphics[]{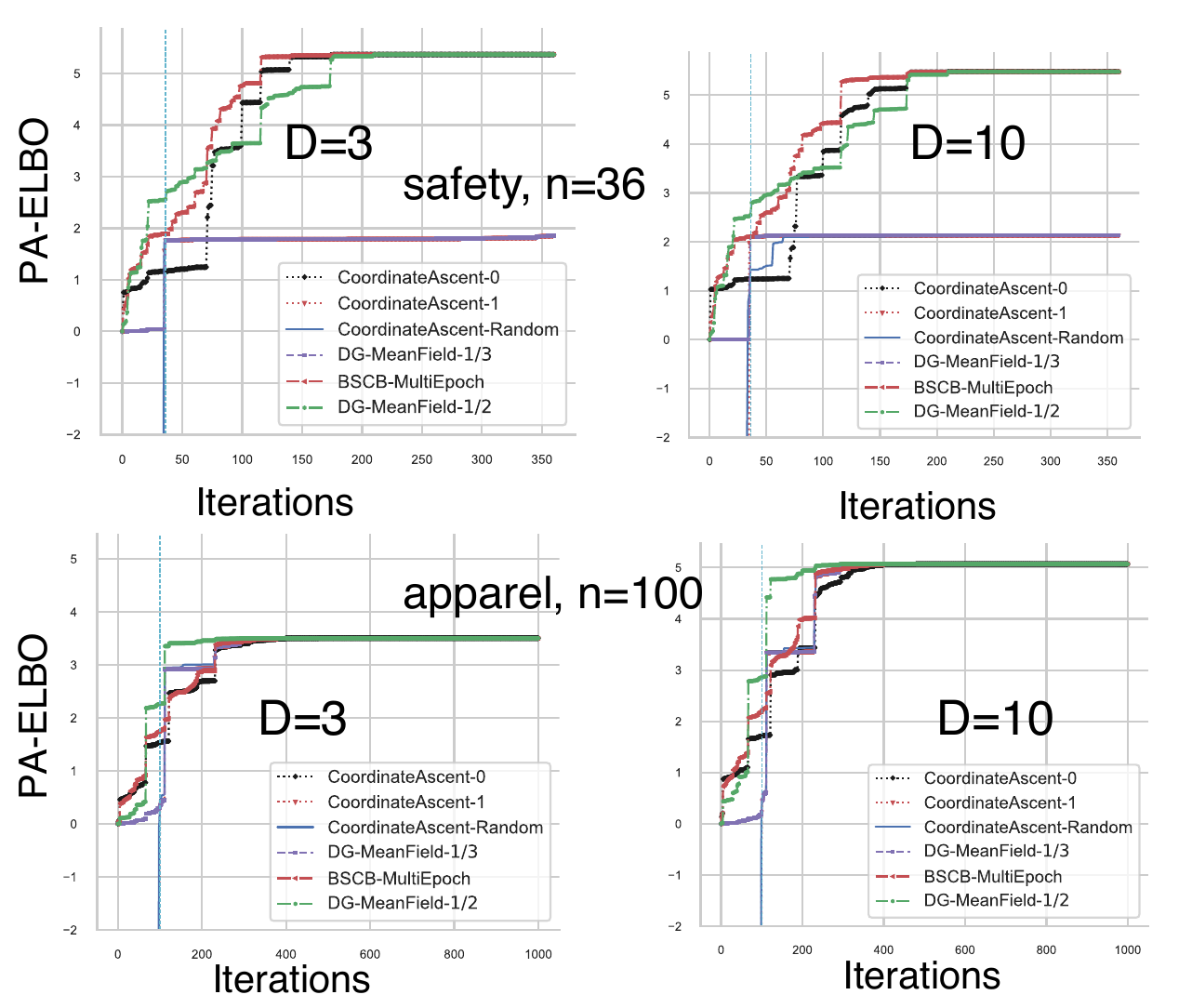}
	\caption{PA-ELBO  on Amazon data.  The figures  trace trajectories of multi-epoch algorithms. Cyan vertical line shows the one-epoch point.
	}
	\label{PA-ELBO_safety}
	\vspace{-.1cm}
\end{figure}
\textbf{Real-world Dataset. }
We  tested the mean field methods on the
trained FLID models from \cite{Tschiatschek16diversity} on Amazon Baby Registries
dataset.
After preprocessing, this dataset
has 13 categories, e.g., ``feeding" \& ``furniture".
One category contains a certain number of registries
over the ground set of this category, e.g., ``strollers" has
5,175 registries with $n=40$.
One can refer to \cref{tab_summ_elbo} for specific dimensionalities
on each of the category\footnote{More details on this dataset can be found in  \citet{gillenwater2014expectation}.}.
For each category, three
classes of
models were trained, with latent dimensions $D= 2, 3, 10$, repectively, on
10 folds of the data.

\subsection{Results on One-epoch Algorithms}

\cref{tab_summ_elbo}  summarizes the outputs
of one-epoch algorithms for both ELBO and PA-ELBO
objectives. For each category, the results
of FLID models with three dimensionalities ($D= 2, 3, 10$)
are reported. 

\paragraph{ELBO Objective.}
The results are summarized in columns 3 to 5 in \cref{tab_summ_elbo}.
 The mean and standard deviation are calculated
for 10 FLID models trained on 10 folds of the data.
One can observe that both \algname{DR-DoubleGreedy} and \algname{BSCB}
improve over the baseline \algname{Submodular-DoubleGreedy}, which has only a 1/3 approximation guarantee.
Furthermore, \algname{DR-DoubleGreedy} generates better solutions
than \algname{BSCB} for almost all of the cases, though they have the same
approximation guarantee.

\paragraph{PA-ELBO objective.}

The results are summarized in columns 6 to 8 in \cref{tab_summ_elbo}.
For each category, out of the 10 folds of data,
we have ${10\choose  2} = 45$ pairs of folds.
The mean and standard deviation are computed 
for these $45$ pairs for each category and each
latent dimensonality $D$.
One can still observe that \algname{DR-DoubleGreedy} and \algname{BSCB} significantly 
improve over \algname{Submodular-DoubleGreedy}. 
Moreover, \algname{DR-DoubleGreedy}  produces better
solutions than \algname{BSCB} in most of the experiments.

\subsection{Results on Multi-epoch Algorithms}

\paragraph{PA-ELBO Objective.}

\cref{PA-ELBO_safety} shows representative
results on PA-ELBO objectives.
One can see that after one epoch, \algname{\dgmf-$1/2$}
almost always returns the best solution.
In most of the experiments, \algname{\dgmf-$1/2$}  was the
fastest algorithm to converge.
However, \algname{\nmf} is quite sensitive to initializations.
After sufficiently many iterations, most  multi-epoch
algorithms converge to similar ELBO values.
This is consistent with the intuition since after one epoch,
all algorithms are using the same strategy: conducting
coordinate-wise maximization. However, for \algname{\nmf} with unlucky initializations, e.g., for category ``safety" (row 1), it may get stuck in poor
local optima. 

The results on ELBO objectives are put 
into \cref{supp_elbo}.

\section{Conclusions}
Probabilistic structured models play an eminent
role in machine learning today, especially models with submodular
costs. Validating such models and their parameters remains an open
issue in applications. We have proposed provable mean field algorithms
for probabilistic log-submodular models and their posterior agreement
score. A novel Double Greedy scheme with optimal $1/2$ approximation
ratio for the general problem of box-constrained continuous
DR-submodular maximization has been proposed and analyzed on
real-world data. 
We plan to generalize the guaranteed mean field
approaches to probabilistic graphical models with a larger class of
energy functions.

\acks{YAB would like to thank Nico Gorbach and Josip Djolonga for fruitful discussions, to thank Sebastian Tschiatschek  for sharing the source code and data.
This research was partially supported by the Max Planck ETH Center for Learning System}

\clearpage
{
\bibliography{\bibpath}
}

\clearpage 
\onecolumn
\appendix
\appendixtitle{Appendix}

\section{There Exist Poor Local Optima}
\label{sec_bad_localoptima}

If one  only assume the objective function $f(\x)$
to be continuous DR-submodular, and considering that 
the multilinear extension of a submodular set function 
is continuous DR-submodular, we can take the 
examples from literatures on combinatorial optimization, e.g., 
\cite{feige2011maximizing}, 
to show that bad  local optima exist. 

Here we provide a \textit{stronger} example,
where we  assume that the objective function $f(\x)$
 has the 
same structure as the ELBO objective \labelcref{opt_problem_meanfield}. And still there exist 
bad local optima. 
These local optima have  arbitrarily
small objective value  compared to the global optimum. 
And \algname{\nmf}
will get stuck in this local optimum without extra techniques. 

Suppose that we have a directed graph $G = (\groundset, A)$ with 
 four vertices, $\groundset = \{ 1,2,3,4 \} $ and four  arcs, $A = \{ (1,2), 
(2,3), (3,2), (3,4)  \}$. The weights of the arcs are (let $b,c$  be
large positive numbers): $w_{1,2} = c$, $w_{2,3} = c$, $w_{3,4} = c$, $w_{3,2} = bc$.
Let 
$F(S)$ denote the sum of weights of arcs
leaving $S$.  Consider its ELBO (using techniques from \cref{gibbs_multilinear}), 
\begin{align}
f(\x)&  = \multi(\x) + \sum_{i \in\groundset} H(x_i) \\  
& =  \sum_{(i,j)\in A} w_{ij} x_i (1- x_j) +  \sum_{i \in\groundset} H(x_i)\\ 
& = c x_1(1-x_2) + c x_2(1-x_3) + c x_3 (1-x_4) + bc x_3 (1-x_2) +  \sum_{i \in\groundset} H(x_i).
\end{align}

Consider  the point $\y = [0.5, 1, 0, 0.5 ]^\trans$,  it has  function value $f(\y) = c + 2\log 2$. Consider a second   point $\bar \x = [ 1,0,1,0 ]^\trans$,  while the global optimum $f(\x^\star)$ must be greater than $f(\bar \x) = (2+b)c$. When $b$ becomes  large, the ratio $\frac{f(\y)}{f(\x^\star)} \leq \frac{c + 2\log 2}{(2+b)c}$ can be arbitrarily small.

{\algname{\nmf}\ may get stuck on the point $\y = [0.5, 1, 0, 0.5 ]^\trans$.}
This can be illustrated by  considering the course  of \algname{\nmf}. 
Suppose wlog. that  \algname{\nmf}\ processes  coordinates in the order of $1\rightarrow 4$ (actually it is
the same with any orders). 

For coordinate 1, 
$\nabla_1 \multi(\x) = c (1-x_2) $, so $\nabla_1 \multi(\y) = 0$, after applying $\sigma(\nabla_1 \multi(\y) )$,
$y_1$ remains to be $0.5$. 

For coordinate 2, $\nabla_2 \multi(\x) = c (1-x_3) - bc x_3 $, so $\nabla_2 \multi(\y) = c$. When $c$ is sufficiently large (approaching infinity), after applying $\sigma(\nabla_2 \multi(\y))$, $y_2$ will still be 1.

For coordinate 3, $\nabla_3 \multi(\x) = - c x_2 + c (1-x_4) +  bc (1-x_2) $, so $\nabla_3 \multi(\y) = -0.5 c$. When $c$ is sufficiently large (approaching infinity), after applying $\sigma(\nabla_3 \multi(\y))$, $y_3$ will still be 0.

For coordinate 4, 
$\nabla_4 \multi(\x) = -c x_3 $, so $\nabla_4 \multi(\y) = 0$, after applying $\sigma(\nabla_4 \multi(\y) )$,
$y_4$ remains to be $0.5$.

\section{Proofs for \algname{DR-DoubleGreedy}}

\subsection{Hardness of Problem \labelcref{opt_problem}}
\label{supp_hardness}

\begin{observation}
The problem of maximizing a generally non-monotone DR-submodular continuous
function subject to box-constraints is NP-hard. Furthermore, there is no $(1/2+\epsilon)$-approximation for any $\epsilon>0$, unless RP = NP.
\end{observation}

The proof is very similar to the that  of \citep[Proposition 5]{bian2017guaranteed}, so we just briefly explain here. 
One observation is that the multilinear extension of
a submodular set function is also continuous DR-submodular, so we can  use the same reduction as in \citep[Proposition 5]{bian2017guaranteed} to
prove the hardness results as above.

\subsection{Detailed  Proof of \cref{thm_cont_doublegreedy} }
\label{supp_proof_dr}

\restatheoremone*

\begin{proof}[Proof of \cref{thm_cont_doublegreedy}]

	Define $\intermed^\pare{k} := (\optcont \vee \x^k) \wedge \y^k$. It is clear
	that $\intermed^\pare{0} = \optcont$ and $\intermed^\pare{n} = \x^\pare{n} = \y ^\pare{n}$. 
	One  can notice that as \cref{alg_cont_doublegreedy} progresses, $\intermed^\pare{k}$
	moves from $\optcont$ to $\x^n$ (or $\y^n$). 
	
	Let $r_a = \frac{\delta_a}{\delta_a + \delta_b}, r_b = 1-r_a$, $u = r_a  u_a + (1-r_a)  u_b$.
	
\restalemmaone*

    \begin{proof}[Proof of \cref{lemma1}]

    One can observe that $\x^\pare{k-1} \leq \y^\pare{k-1}$, so from DR-submodularity: $f(\sete{x^\pare{k-1}}{\ele_k}{u_b}) - f(\x^\pare{k-1}) \geq f(\sete{y^\pare{k-1}}{\ele_k}{u_b}) - f(\sete{y^\pare{k-1}}{\ele_k}{a_{\ele_k}}) \geq -\frac{\delta}{n}$.

	Similarly,  because of $\x^\pare{k-1} \leq \y^\pare{k-1}$ and $u_a \leq b_{\ele_k}$, 	from DR-submodularity: $f(\sete{y^\pare{k-1}}{\ele_k}{u_a}) - f(\y^\pare{k-1} )  \geq  f(\sete{x^\pare{k-1}}{\ele_k}{u_a}) - f(\sete{x^\pare{k-1}}{\ele_k}{b_{\ele_k}} )  \geq -\frac{\delta}{n}  $. 
    \end{proof}
	
\restalemmatwo*

	\begin{proof}[Proof of \cref{lemma2}]

		Step I: 
		
		Let us try to lower bound the RHS of \cref{lemma2}.
		
		\begin{align}\notag 
		f(\x^\pare{k}) - f(\x^\pare{k-1}) & = f(\sete{x^\pare{k-1}}{\ele_k}{r_a u_a + r_b u_b}) - f(\x^\pare{k-1}) \\ \notag
		& \overset{\textcircled{1}}{\geq}   r_a f(\sete{x^\pare{k-1}}{\ele_k}{u_a}) + r_b f(\sete{x^\pare{k-1}}{\ele_k}{u_b}) - f(\x^\pare{k-1})\\\notag
		& = r_a [f(\sete{x^\pare{k-1}}{\ele_k}{u_a}) - f(\x^\pare{k-1})] + r_b [f(\sete{x^\pare{k-1}}{\ele_k}{u_b}) - f(\x^\pare{k-1})]\\ \notag
		&  \overset{\textcircled{2}}{\geq}  r_a \delta_a  - r_b \frac{\delta}{n},
		\end{align}  
		where \textcircled{1} is because of that $f$ is concave along one coordinate, 
		\textcircled{2} is from \cref{lemma1}.

		Similarly,
		\begin{align}\notag
		f(\y^\pare{k}) - f(\y^\pare{k-1}) & = f(\sete{y^\pare{k-1}}{\ele_k}{r_a u_a + r_b u_b}) - f(\y^\pare{k-1}) \\ \notag
		& \geq  r_a f(\sete{y^\pare{k-1}}{\ele_k}{u_a}) + r_b f(\sete{y^\pare{k-1}}{\ele_k}{u_b}) - f(\y^\pare{k-1})\\\notag
		& = r_a [f(\sete{y^\pare{k-1}}{\ele_k}{u_a}) - f(\y^\pare{k-1})] + r_b [f(\sete{y^\pare{k-1}}{\ele_k}{u_b}) - f(\y^\pare{k-1})]\\ \notag
		& \geq - r_a \frac{\delta}{n} + r_b \delta_b. 
		\end{align}

		So it holds that
		\begin{align}\label{eq_step1}
		f(\x^\pare{k}) - f(\x^\pare{k-1}) +  f(\y^\pare{k}) - f(\y^\pare{k-1}) \geq r_a\delta_a + r_b\delta_b - \frac{\delta}{n} =  \frac{\delta_a^2 + \delta_b^2}{\delta_a + \delta_b} - \frac{\delta}{n}. 
		\end{align}

		Step II: 
		
		Now let us upper bound the LHS of \cref{lemma2}.
		
		Notice that  $\intermed^\pare{k-1} := (\optcont \vee \x^\pare{k-1}) \wedge \y^\pare{k-1}$.
		For $\intermed^\pare{k-1}$, its $\ele_k$-th coordinate is $x^\star_{\ele_k}$. From $\intermed^\pare{k-1}$ to $\intermed^\pare{k}$, its $\ele_k$-th coordinate
		changes to be $u$. So,
		
		\begin{align}
		f(\intermed^\pare{k-1}) - f(\intermed^\pare{k}) & = f( \sete{\intermed^\pare{k-1}}{\ele_k}{x^\star_{\ele_k}} )  - f( \sete{\intermed^\pare{k-1}}{\ele_k}{u} )
		\end{align}
		Let us consider the following two  situations: 
		
		\begin{enumerate}
			\item   $x^\star_{\ele_k} \leq u$. 
			
			In this case: 
			\begin{align}\notag
			& f(\intermed^\pare{k-1}) - f(\intermed^\pare{k})\\\notag
			& = f( \sete{\intermed^\pare{k-1}}{\ele_k}{x^\star_{\ele_k}} )  - f( \sete{\intermed^\pare{k-1}}{\ele_k}{u} ) \\\notag 
			& \overset{\textcircled{3}}{\leq}   f( \sete{y^\pare{k-1}}{\ele_k}{x^\star_{\ele_k}} )  - f( \sete{y^ \pare{k-1}}{\ele_k}{u} )    \\  \notag
			& =    f( \sete{y^\pare{k-1}}{\ele_k}{x^\star_{\ele_k}} )  - f( \sete{y^ \pare{k-1}}{\ele_k}{r_au_a + r_b u_b} )     \\ \notag  
			&  \overset{\textcircled{4}}{\leq}  r_a [f( \sete{y^\pare{k-1}}{\ele_k}{x^\star_{\ele_k}} )  - f( \sete{y^ \pare{k-1}}{\ele_k}{u_a} )]  + r_b [f( \sete{y^\pare{k-1}}{\ele_k}{x^\star_{\ele_k}} )  - f( \sete{y^ \pare{k-1}}{\ele_k}{u_b} ) ] \\ \notag
			&  \leq  r_a [f( \sete{y^\pare{k-1}}{\ele_k}{x^\star_{\ele_k}} )  - f( \sete{y^ \pare{k-1}}{\ele_k}{u_a} )] + r_b \frac{\delta}{n} \quad \text{(selection rule of \cref{alg_cont_doublegreedy})} \\\notag
			& \overset{\textcircled{5}}{\leq} r_a [f( \sete{y^\pare{k-1}}{\ele_k}{u_b} ) + \frac{\delta}{n}  - (f( \y^\pare{k-1})  - \frac{\delta}{n} )] + r_b \frac{\delta}{n}  \\\notag
			& \leq r_a \delta_b  + (2r_a + r_b)\frac{\delta}{n}, 	 
			\end{align}
			where \textcircled{3} is because $\intermed^\pare{k-1} \leq \y^\pare{k-1}$ and
			DR-submodularity of $f$, \textcircled{4} is from concavity of $f$ along one coordinate, \textcircled{5} is because of the selection rule of \cref{alg_cont_doublegreedy} and \cref{lemma1}.

			\item $x^\star_{\ele_k} >  u$: 
			
			In this case: 
			\begin{align}\notag
			& f(\intermed^\pare{k-1}) - f(\intermed^\pare{k})\\\notag
			& = f( \sete{\intermed^\pare{k-1}}{\ele_k}{x^\star_{\ele_k}} )  - f( \sete{\intermed^\pare{k-1}}{\ele_k}{u} ) \\ \notag
			& \leq   f( \sete{x^\pare{k-1}}{\ele_k}{x^\star_{\ele_k}} )  - f( \sete{x^ \pare{k-1}}{\ele_k}{u} )    \text{\quad ($\intermed^\pare{k-1}\geq \x^\pare{k-1}$ \& DR-submodularity)}   \\ \notag 
			& =    f( \sete{x^\pare{k-1}}{\ele_k}{x^\star_{\ele_k}} )  - f( \sete{x^ \pare{k-1}}{\ele_k}{r_au_a + r_b u_b} )     \\ \notag  
			&  \leq  r_a [f( \sete{x^\pare{k-1}}{\ele_k}{x^\star_{\ele_k}} )  - f( \sete{x^ \pare{k-1}}{\ele_k}{u_a} )]  + r_b [f( \sete{x^\pare{k-1}}{\ele_k}{x^\star_{\ele_k}} )  - f( \sete{x^ \pare{k-1}}{\ele_k}{u_b} ) ] \\ \notag
			&  \leq r_a \frac{\delta}{n} +  r_b [f( \sete{x^\pare{k-1}}{\ele_k}{x^\star_{\ele_k}} )  - f( \sete{x^ \pare{k-1}}{\ele_k}{u_b} ) ] \\\notag
			&   \leq r_a \frac{\delta}{n} +   r_b  [  ( f( \sete{x^\pare{k-1}}{\ele_k}{u_a} ) + \frac{\delta}{n} )   - ( f(\x^\pare{k-1})  - \frac{\delta}{n} )  ] \\\notag
			& = r_b \delta_a + ( 2r_b + r_a)\frac{\delta}{n} 	
			\end{align}

		\end{enumerate}
	We can conclude that in both the above cases, it holds that
	\begin{align}\label{eq_step2}
		f(\intermed^\pare{k-1}) - f(\intermed^\pare{k}) \leq \frac{\delta_a \delta_b}{\delta_a + \delta_b} + \frac{2\delta}{n}
	\end{align}	
		
Combining \cref{eq_step1} and \cref{eq_step2} we can get,
\begin{align}
	\frac{1}{2} [f(\x^\pare{k}) - f(\x^\pare{k-1}) +  f(\y^\pare{k}) - f(\y^\pare{k-1})] \geq f(\intermed^\pare{k-1}) - f(\intermed^\pare{k})   - \frac{2.5\delta}{n}
\end{align}
Thus we reach 	\cref{lemma2}.	
		
	\end{proof}
	
	Now we can finalize the proof. For \cref{lemma2}, let 
	us sum for $k = 1,...,n$, we can get, 
	\begin{align}
		f(\optcont) - f(\x^\pare{n}) \leq \frac{1}{2} [ f(\x^\pare{n}) - f(\a) + f(\y^\pare{n}) - f(\b)  ] + 2.5\delta 
	\end{align}
	After rearrangement, 
one can show that  
	$ f(\x^n) \geq  \frac{1}{2}  f(\optcont)  + \frac{1}{4}  [f(\a) + f(\b)] - \frac{5\delta}{4}$.
	
\end{proof}

\section{Mean Field Lower Bounds for PSMs}
\label{sec_meanfield_lowerbounds}

Log-submodular models \citep{djolonga14variational} are a class of
probabilistic point processes  over subsets of a ground set $\groundset = [n]$,
where the log-densities are submodular set functions $F(S)$:
$p(S) = \frac{1}{\parti}\exp(F(S))$, where 
$\parti = \sum_{S\subseteq \groundset}\exp(F(S))$
is the  partition function.

Mean-field inference aims to approximate $p(S)$ by
a fully factorized product  distribution
$q(S|{\x}):= \prod_{i\in S}x_i \prod_{j\notin S}(1-x_j), \x\in
[0,1]^n$,
by minimizing the distance measured w.r.t. the Kullback-Leibler
divergence between $q$ and $p$, i.e.,
$\kl{q}{p} =  \sum_{S\subseteq \groundset} q(S|{\x})
\log\frac{q(S|{\x})}{p(S)}$. $\kl{q}{p}$ is non-negative,
so 
\begin{align}\notag 
& 0\leq \kl{q}{p} =  \sum_{S\subseteq \groundset} q(S|{\x})
\log\frac{q(S|{\x})}{p(S)}\\
& = -  \E_{q(S|{\x})} [\log p(S)] + \entropy{q(S|{\x})}  \\   
&= 
  -\sum_{S\subseteq \groundset}  F(S) \prod_{i\in S}x_i \prod_{j\notin S}(1-x_j)+ \sum\nolimits_{i=1}^{n} [x_i\log x_i + (1-x_i)\log(1-x_i)] + \log \parti.
\end{align}
where $\entropy{\cdot}$ is the entropy. So one can get 
$\log \parti \geq \sum_{S\subseteq \groundset}  F(S) \prod_{i\in S}x_i \prod_{j\notin S}(1-x_j)- \sum\nolimits_{i=1}^{n} [x_i\log x_i + (1-x_i)\log(1-x_i)]  = (\text{ELBO})$.

Multilinear extension $\multi(\x)$ of 
a submodular set function is continuous DR-submodular \cite{bach2015submodular}, 
and $- \sum\nolimits_{i=1}^{n} [x_i\log x_i + (1-x_i)\log(1-x_i)] $
is seperable and concave on each coordinate, so 
$(\text{ELBO})$ is DR-submodular w.r.t. $\x$.
Maximizing $(\text{ELBO})$ amounts to 
minimizing the Kullback-Leibler divergence.

For \text{(PA-ELBO)}  \labelcref{pa_elbo} , it is the sum of two multilinear extensions (weighted by $\beta >0$) and the binary entropy term, since the non-negative
sum of two DR-submodular functions is still DR-submodular, so 
\text{(PA-ELBO)} in \labelcref{pa_elbo} is also continuous DR-submodular. Thus it fits into the general optimization problem \labelcref{opt_problem}.

\section{Full Lower Bounds of PA Objective}
\label{upperbounds_pa}

By giving  upper bounds for 
$\log \parti(\beta; \data') + \log \parti(\beta; \data'')$,
we can get the full lower bounds of the PA objective.

Let us take one $\log \parti(\beta; \data')$ for example. 
This  can be achieved using techniques from \citep{djolonga14variational}, which is done by 
 optimizing  supergradients \citep{djolonga14variational} of  $F(S\mid \data')$.
A representative supergradient
is the bar supergradient, which is defined as: if $i\in A$, $\bar \s^A = F_{\groundset -\{i\}} (\{i \}\mid \data' )$, if $i\notin A$, $\bar \s^A = F (\{i \} \mid \data')$, 
where $F_{B} (A\mid \data')$ is the marginal
gain of $A$ based on $B$. 
 Then, 
\begin{align}
\log \parti(\beta; \data') \leq \min _{A} \log \parti^+(\bar \s^A, F(A\mid \data') - \bar \s^A(A) )  = \min _{A} F(A\mid \data') + \m (A\mid \data'),
\end{align}
where $\m ({\{i\} \mid \data'}) = \log ( 1 + e^ { -F_{\groundset -\{i\}} (\{i \}\mid \data') }) - \log ( 1 + e^{F(\{i\} \mid \data')})$. 

So the full lower bound of  PA objective in  \labelcref{pa_objective} is, 
\begin{align}\label{}
&  \mathrm{log} \sum\nolimits_{S \subseteq \groundset} p_{\beta}(S \mid \data') p_{\beta}(S \mid \data'')   \quad \text{(\palong\  objective)} \\ 
&{=} -\left[ \sum\nolimits_{S \subseteq \groundset} 	q(S|\x) \right] \mathrm{log} \frac{\sum_{S\subseteq \groundset} q(S|\x)}{\sum\nolimits_{S \subseteq \groundset} p_\beta (S \mid \data') p_\beta (S \mid \data'')} \nonumber \\\notag 
&\stackrel{\text{log-sum inequality}}{\geq} -\sum\nolimits_{S \subseteq \groundset} q(S|\x) \ \mathrm{log} \frac{q(S|\x)}{p_\beta(S \mid \data') p_\beta(S \mid \data'')} 
= \H(q) + \mathbb{E}_q \mathrm{log} \ p_\beta(S \mid \data') +\mathbb{E}_q \log \ p_\beta(S\mid \data'') \\\notag 
&  = \underbrace{\H(q) + \beta \ \mathbb{E}_q F(S|\data') +\beta \ \mathbb{E}_q F(S|\data'')}_{\text{(PA-ELBO) in \labelcref{pa_elbo}}}  \!   -\log \parti(\beta; \data') - \!\log \parti(\beta; \data'')  \\
& \geq \max_{q}  \underbrace{\H(q) + \beta \ \mathbb{E}_q F(S|\data') +\beta \ \mathbb{E}_q F(S|\data'')}_{\text{(PA-ELBO) in \labelcref{pa_elbo}}} - 
 \min _{A} \left[F(A\mid \data') + \m (A\mid \data')\right] -  \min _{A} \left[F(A\mid \data'') + \m (A\mid \data'')\right]
\end{align}

\section{Detailed Multilinear Extension in Closed Form}
\label{supp_multilinear_closedform}

\subsection{More on Sampling}

\begin{lemma}[Hoeffding Bound, Theorem 2 in \cite{hoeffding1963probability}]

	Let $X_1, ..., X_m$ be independent random variables
	such that for each $i, a \leq X_i \leq  b$, with $a, b\in 
	\R$. Let $\bar X  = \frac{1}{m} \sum_{i=1}^{m} X_i$. Then
	\begin{align}
			\Pr [|\bar{X} - \E(X)| > t]\leq  e^ {-
			\frac{2t^2m}{(b - a)^2}}.		
	\end{align}
\end{lemma}
According to the  Hoeffding bound \citep{hoeffding1963probability}, one can
 easily derive that $\frac{1}{k}\sum_{i=1}^{k} F(S_i)$ is arbitrarily close to $\multi(\x)$ with increasingly 
 more samples: With probability at least $1- e^{-k\epsilon^2/2}$, it holds that $|\frac{1}{k}\sum_{i=1}^{k} F(S_i) - \multi(\x)| \leq \epsilon \max_S |F(S)| $, for all $\epsilon > 0$.

\subsection{Some Gibbs Random Fields}

\paragraph{Undirected  \maxcut.}
For \maxcut, its objective is $F(\v) = \frac{1}{2}\sum_{(i,j)\in E} w_{ij} (v_i + v_j -2v_i v_j), \v \in \{0,1\}^\groundset  $.  Its 
multilinear extension is $\multi(\x) = \frac{1}{2}\sum_{(i,j)\in E} w_{ij} (x_i + x_j -2x_i x_j), \x \in [0,1]^\groundset $.

\paragraph{Directed \maxcut.} Its objective is $F(\v) = \sum_{(i,j)\in E} w_{ij} v_i (1- v_j), \v \in \{0,1\}^\groundset  $.
Its 
multilinear extension is $\multi(\x) = \frac{1}{2}\sum_{(i,j)\in E} w_{ij} x_i (1- x_j),  \x \in [0,1]^\groundset $.

\textbf{Ising models.}
For Ising models with non-positive pairwise interactions, $F(\v) = \sum_{s\in \groundset} \theta_s v_s + \sum_{(s,t)\in E} \theta_{st}v_s v_t$,  $\v\in \{0, 1 \}^\groundset$, 
this objective can be easily verified to
be submodular. 
Its multilinear extension is:
\begin{align}
\multi(\x)= \sum_{s\in \groundset} \theta_s x_s + \sum_{(s,t)\in E} \theta_{st}x_s x_t 
\end{align}
Lower bound of its log-partition function is 
$\multi(\x) + \sum_{s\in \groundset} H_s(x_s), \x \in [0, 1]^\groundset$. 
When updating $x_s$ and fix all other coordinates, it is easy to see that
\begin{align}
x_s \leftarrow \sigma( \theta_s + \sum_{t\in N(s)} \theta_{st} x_t ), 
\end{align}
where  $N(s)$ are the neighbors of $s$.  

\subsection{More on FLID-style Objectives}

The more refined way to compute partial directives can be expressed by considering the following derivation, 
\begin{align}\notag 
& \nabla_i  \left( \multi(\x) - \sum_ {i\in [n]} u'_i x_i\right)\\\notag
& =  \sum_{S\subseteq \groundset, S\ni i } F(S) \prod_{j \in S\backslash\{i\}}x_j \prod_{j'\notin S}(1-x_{j'}) - \sum_{S\subseteq \groundset\backslash \{i\} }\ F(S) \prod_{j\in S} x_j \prod_{j'\notin S, j'\neq i} (1-x_{j'})\\\notag
& =  \sum_{d=1}^{D}  \left[\sum_{S\subseteq \groundset, S\ni i } \max_{i \in S} W_{i, d} \prod_{j \in S\backslash\{i\}}x_j \prod_{j'\notin S}(1-x_{j'}) - \sum_{S\subseteq \groundset\backslash \{i\} }\ \max_{i \in S} W_{i, d} \prod_{j\in S} x_j \prod_{j'\notin S, j'\neq i} (1-x_{j'})\right]\\\notag
& = \sum_{d=1}^{D}  [ W_{i_d(l_i), d}  \prod_{m=l_i+1} ^n (1-x_{m}) +  \sum_{l=l_i+1}^n  W_{i_d(l), d} x_{i_d(l)} \prod_{m=l+1} ^n (1-x_{m}) \\ \notag
& - \sum_{l = 1}^{l(i)}  W_{i_d(l), d} x_{i_d(l)} \prod_{m=l+1, m\neq l(i)} ^n (1-x_{m})   -   \sum_{l = l(i) + 1}^{n}  W_{i_d(l), d} x_{i_d(l)} \prod_{m=l+1} ^n (1-x_{m})]
\end{align}

\subsection{Approximation for Concave Over Modular Functions}
A general form is,
\begin{align}\notag
F(S) & =  \sum_{j=1}^{M} w_j \psi ( \m^j(S) ) \\\notag
& =  \sum_{j=1}^{M} w_j  [ \m^j(S) ]^a. 
\end{align}
$\psi()$ is a concave function, a common choice 
is $\psi(y) = y^a, a\in (0, 1] $. 
A simple approximation is  $\hat F(S) =  \sum_{j=1}^{M} w_j  \sum_{i\in S}  (m^j_i)^a$,
which approximates $F(S)$ up to a factor of $\bigo{|S|^{1-a}}$ \cite{iyer2015submodular}. 
Since $\hat F(S) $ is modular, one can directly get its multilinear 
extension.

\section{More Experimental Results}
\label{supp_moreexp}

We put more results in this section. 
It includes
experiments on both synthetic datasets and
real-world datasets.

\subsection{ELBO Objective}
\label{supp_elbo}

\setkeys{Gin}{width=0.8\textwidth}
\begin{figure}[ht]
	\center 
	\small
	\includegraphics[]{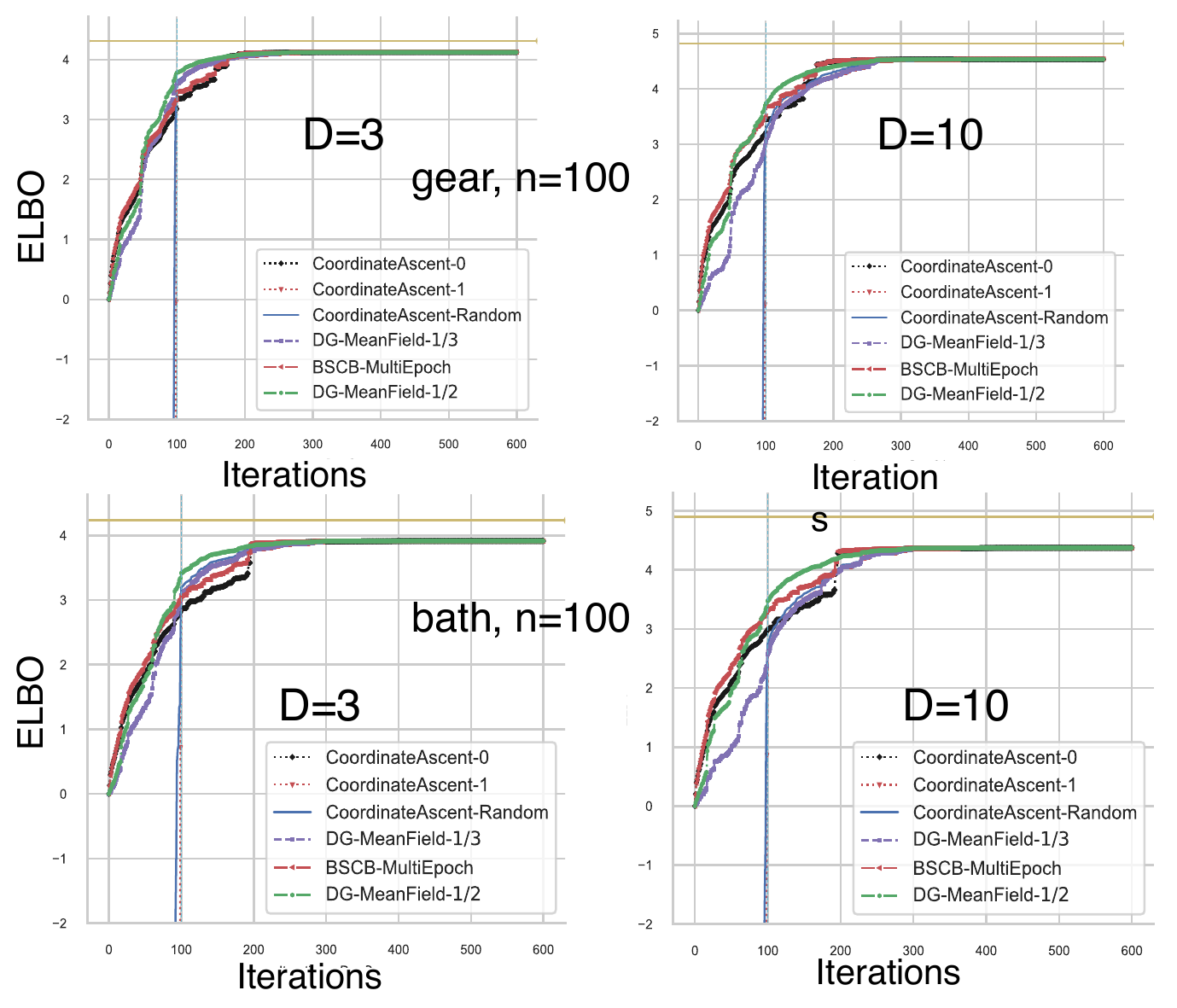}
	\caption{Typical trajectories of multi-epoch algorithms on ELBO objective  for Amazon data. 1st row:  ``gear";
		2nd row: ``bath".
		Cyan vertical line shows the one-epoch point. Yellow line shows  the true log-partition.}
	\label{fig3}
	\vspace{-0.3cm}
\end{figure}
\cref{fig3} records typical trajectories of multi-epoch algorithms for ELBO objectives.
Note that the cyan vertical lines indicate the one-epoch
point.
It shows that after one epoch, \algname{\dgmf-$1/2$}
almost always returns the best solution, and
it  is also the fastest  one to converge.
However, \algname{\nmf} is quite sensitive to initializations.
After sufficiently many iterations, all  multi-epoch
algorithms converge to similar ELBO values.
This is consistent with the intuition because  after one epoch,
all algorithms are  conducting
coordinate-wise maximization.
One can also observe that the obtained ELBO is
close to the true log partition functions (yellow lines).

\subsection{Experiments on \algname{Shrunken Frank-Wolfe}}

Though \algname{shrunken FW} method is not only computationally too expensive, but also have worse approximation guarantee, we still would like 
to see whether it would produces good solution with
more computational resources. In
order to verify this, we run all multi-epoch
algorithms for 6 epochs, while run \algname{shrunken FW} for 60 epochs, results  are shown in the figure bellow: even with {10 times more computations}, \algname{shrunken FW} still performs worse than the  proposed algorithm \algname{DG-MeanField-1/2}. Sometimes \algname{shrunken FW}  has
comparable performance with coordinate descent variant.
\setkeys{Gin}{width=1\textwidth, height=0.19\textwidth}
\begin{figure}[ht]
	\includegraphics[]{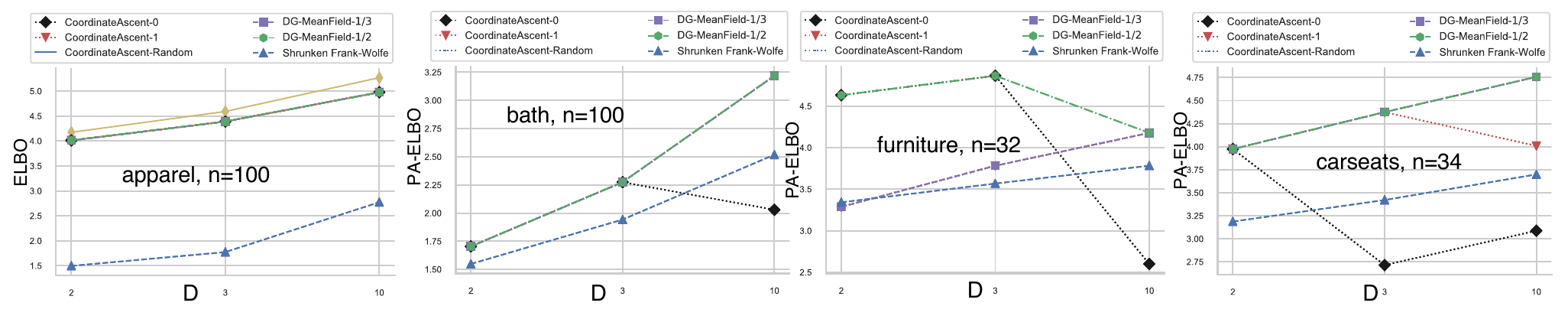}
	\label{fig1}
\end{figure}

\subsection{Synthetic Results}

\setkeys{Gin}{width=0.35\textwidth, height=0.31\textwidth}
\begin{figure}[!ht]
	\footnotesize  
	\subfloat[$n=10$ \label{fig_}]{%
		\includegraphics[]{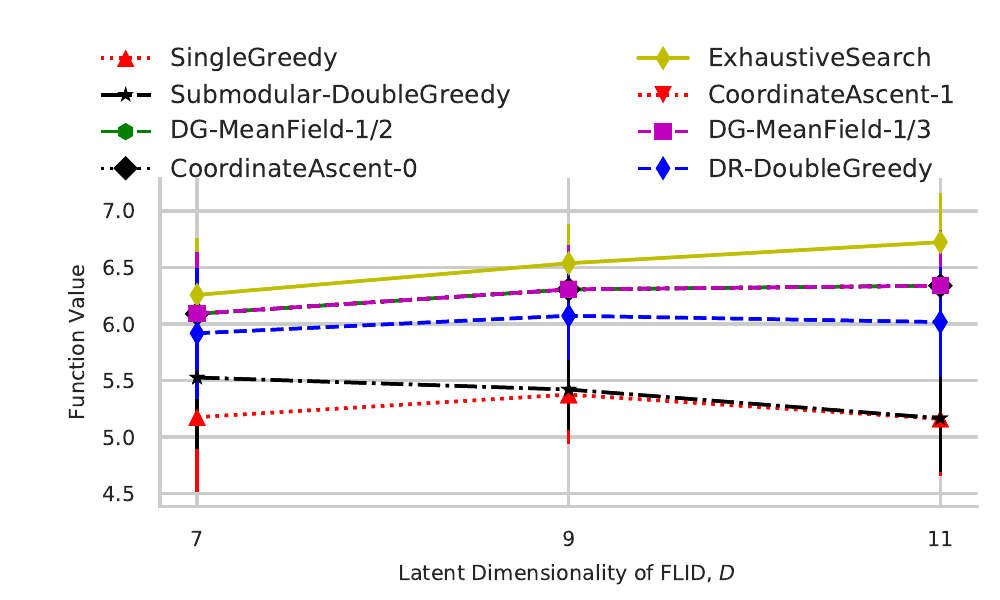}
	}
	\subfloat[$D=7$ \label{fig_}]{%
		\includegraphics[]{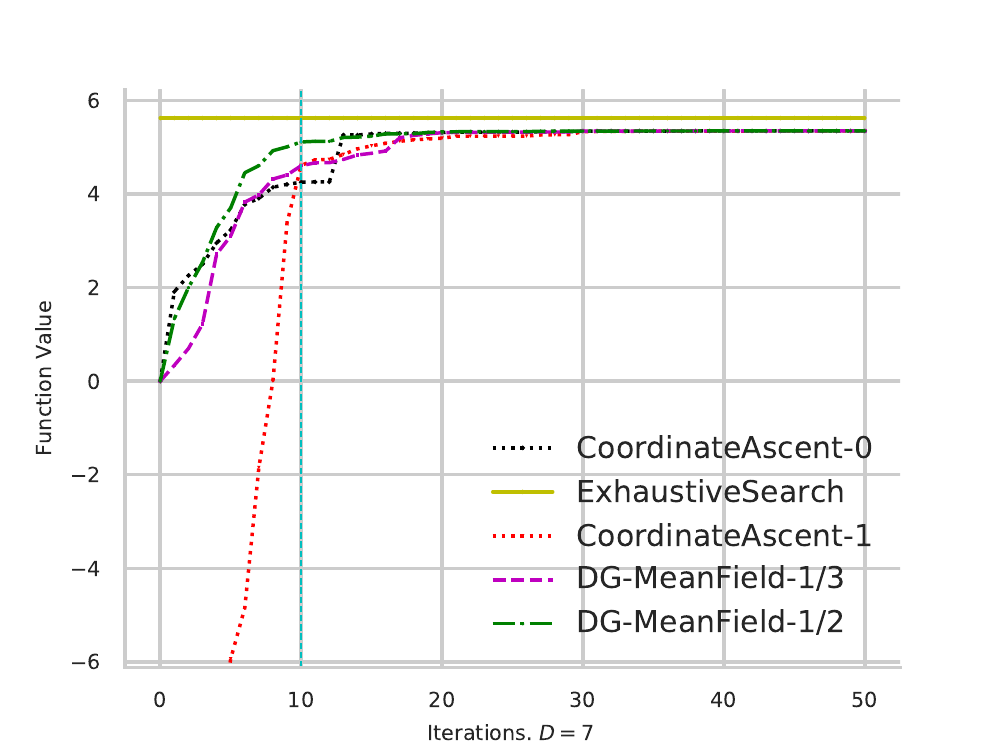}
	}
	\subfloat[$D=11$ \label{fig_}]{%
		\includegraphics[]{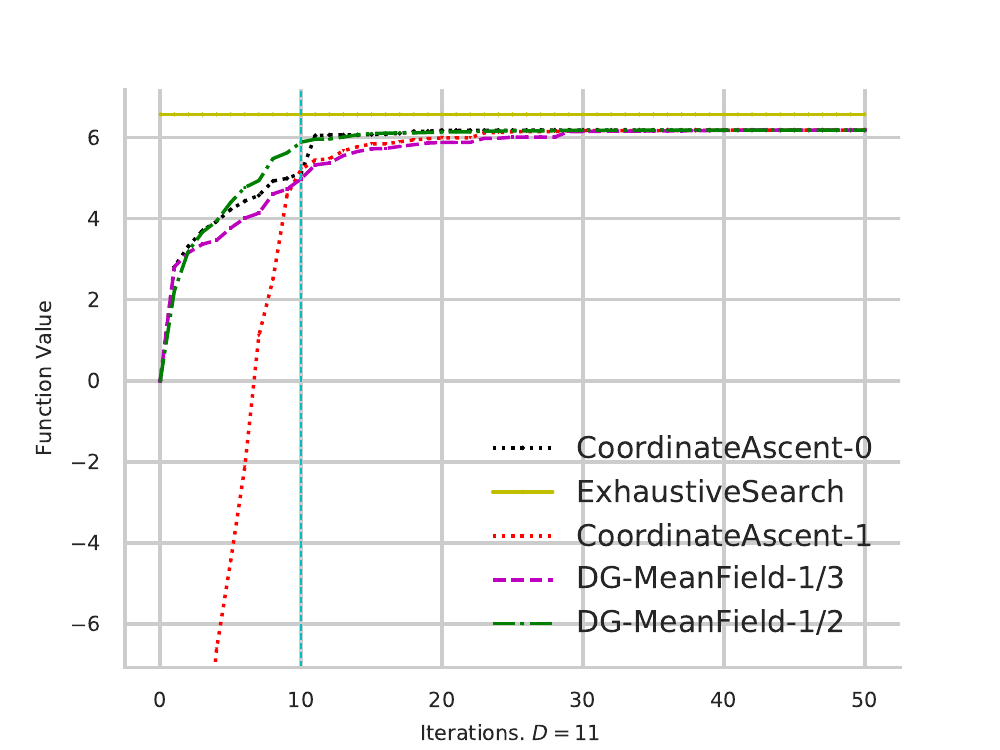}
	}
	\vspace{-.3cm}
	\\
	\subfloat[$n=14$ \label{fig_}]{%
	\includegraphics[]{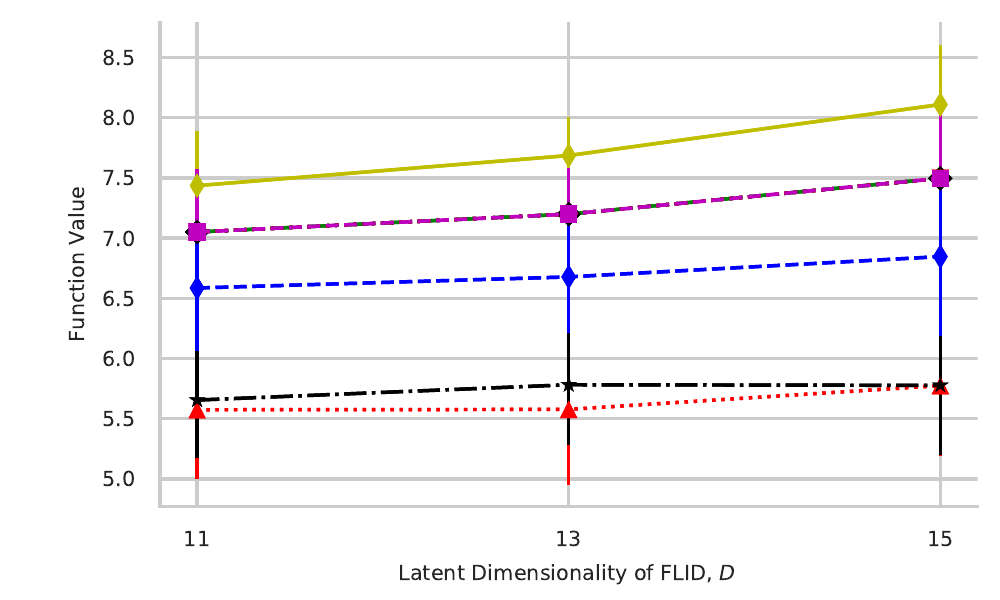}
}
\subfloat[$D=13$ \label{fig_}]{%
	\includegraphics[]{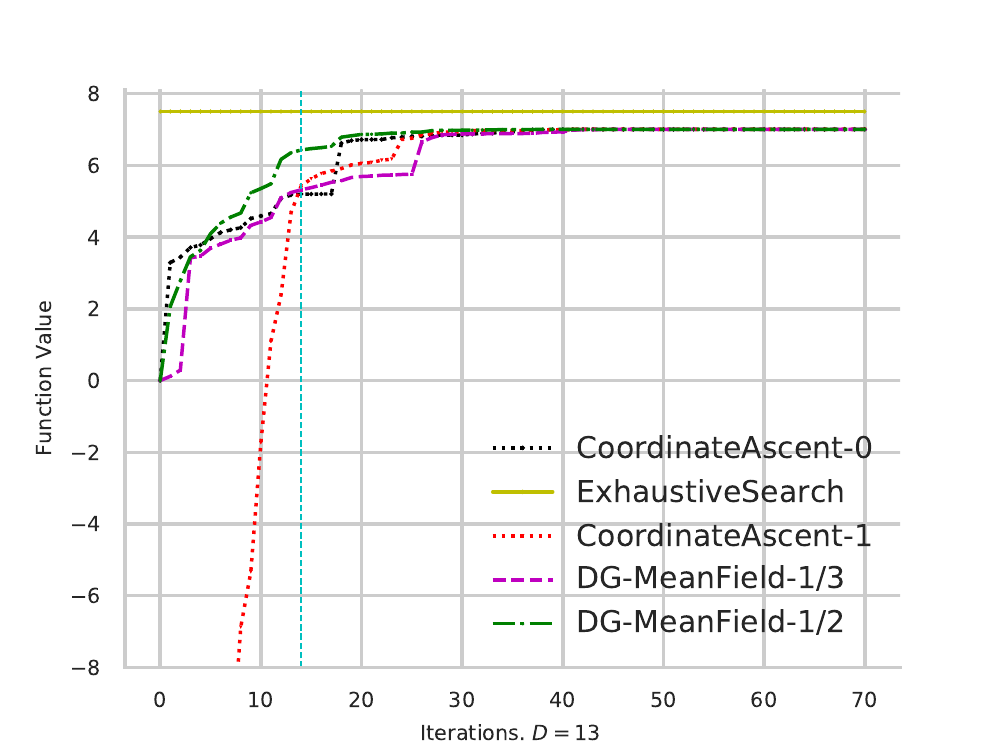}
}
\subfloat[$D=15$ \label{fig_}]{%
	\includegraphics[]{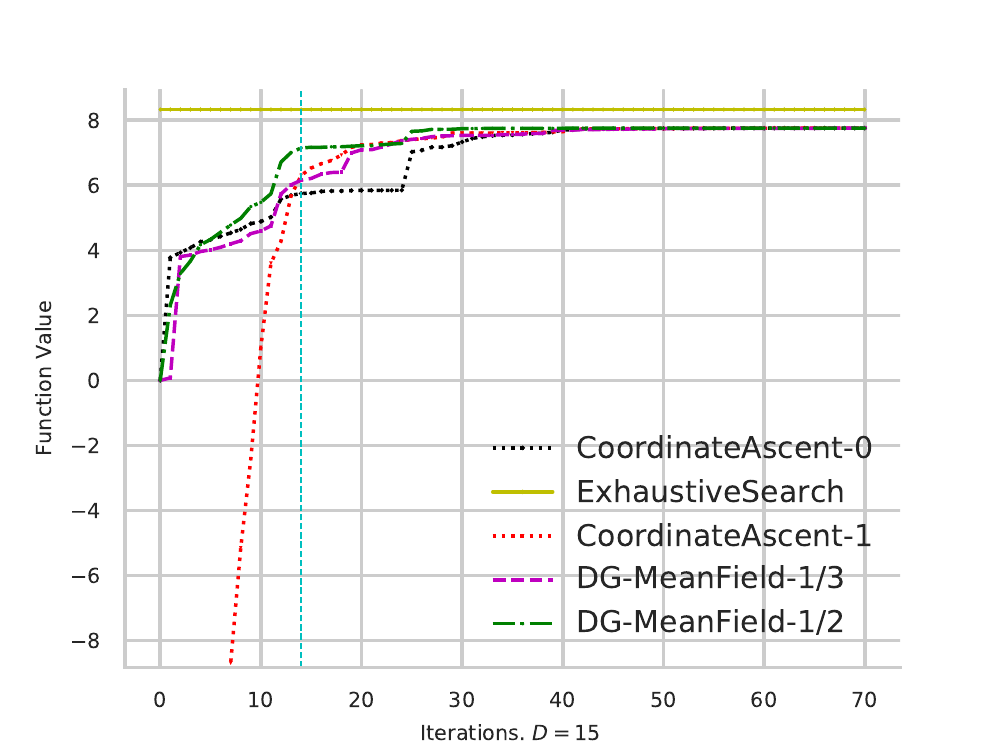}
}
\\
	\subfloat[$n=18$ \label{fig_}]{%
		\includegraphics[]{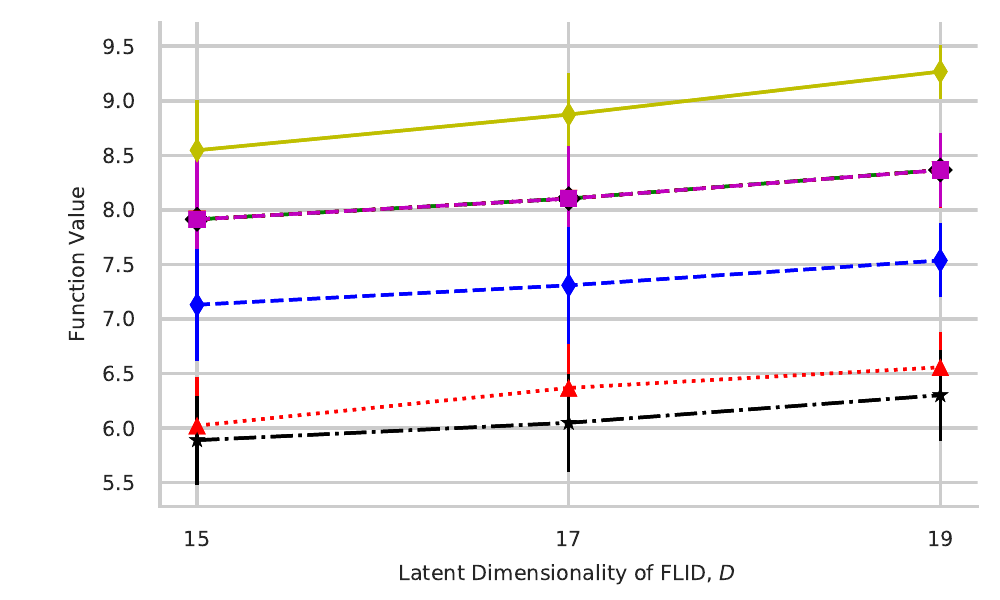}
	}
	\subfloat[$D=17$ \label{fig_}]{%
		\includegraphics[]{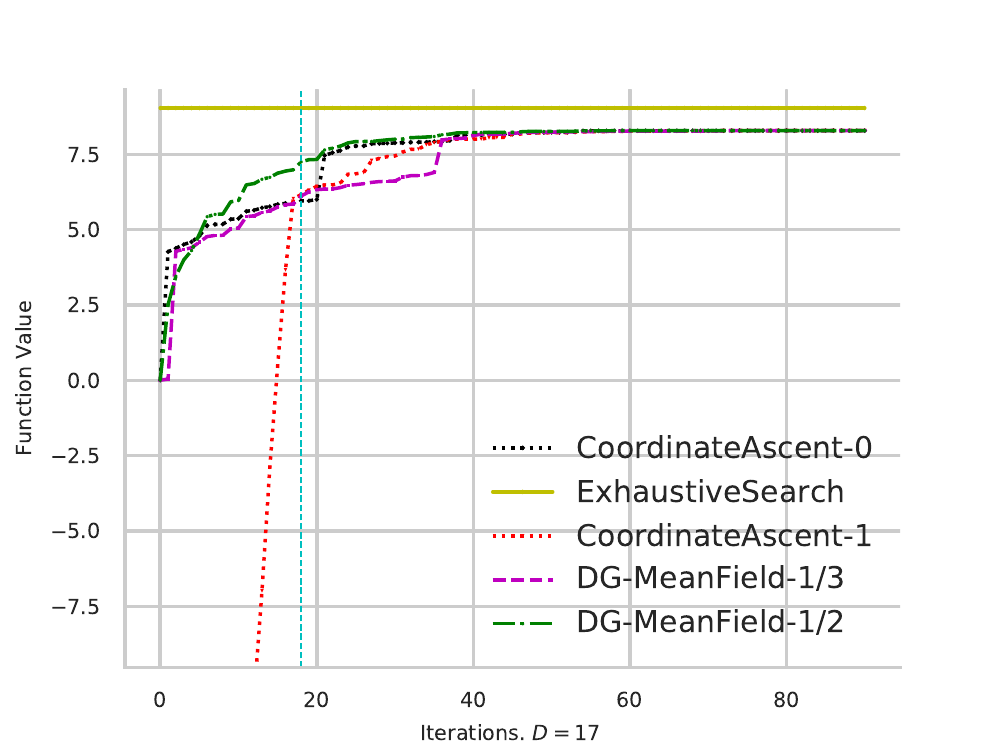}
	}
	\subfloat[$D=19$ \label{fig_}]{%
		\includegraphics[]{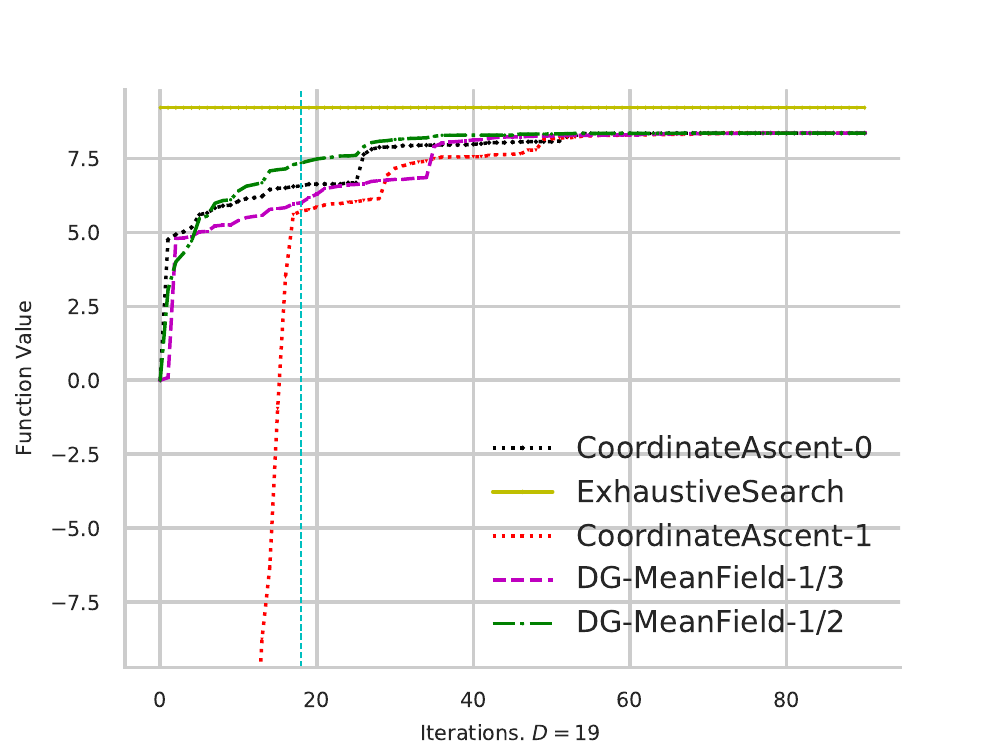}
	}
	\\
	\caption{FLID synthetic results. First column shows the function values returned by different algorithms. The other columns show trajectories of multi-epoch algorithms.  Cyan vertical line shows one-epoch. Yellow line indicates the true log-partition.}
	\label{syn_uD}
\end{figure}

We generate FLID models in the following  manner:  We firstly generate 
the latent representation matrix $\BW\in \R^{n\times D}$ such that
each entry of $W_{i,d} \sim U(0,1)$. It is clear that for FLID, $F(\emptyset)=0$. We then 
set $\u$ to be  proportional to 
$D$  in a random way  $\u = 0.1D*\one * U(0,1)$ so the objective  is non-monotone. 
\cref{syn_uD} records the results: one row
corresponds to the results for a specific $n$. 
First column is the function value returned
by the algorithms, which are  the average of 10 repeated experiments. 
The other columns are trajectories of multi-epoch
algorithms, since behavior is similar for different 
repeated experiments, we plot the first one here. 
Yellow lines are the true log-partition
returned by exhaustive search, cyan vertical
lines shows the one-epoch point. 
One can  see that 
for one-epoch algorithms, \algname{DR-DoubleGreedy}
returns the highest value. 
For multi-epoch algorithms, 
\algname{\dgmf-$1/2$} is the fasted one to converge. After sufficiently many
epoches, the three multi-epoch algorithms
converge to  solutions with similar function value. 

\subsection{More Results on ELBO Objective}

See \cref{fig_supp_elbo} for more results on the ELBO
objective from Amazon data. 

\setkeys{Gin}{width=0.34\textwidth}
\begin{figure}[ht]
	\subfloat[Legend for subfigs (b,d,g) \label{fig_}]{%
	\includegraphics[]{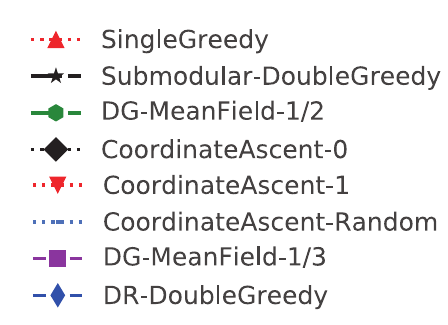}
}
	\subfloat[``strollers", $n=40$ \label{fig_}]{%
		\includegraphics[]{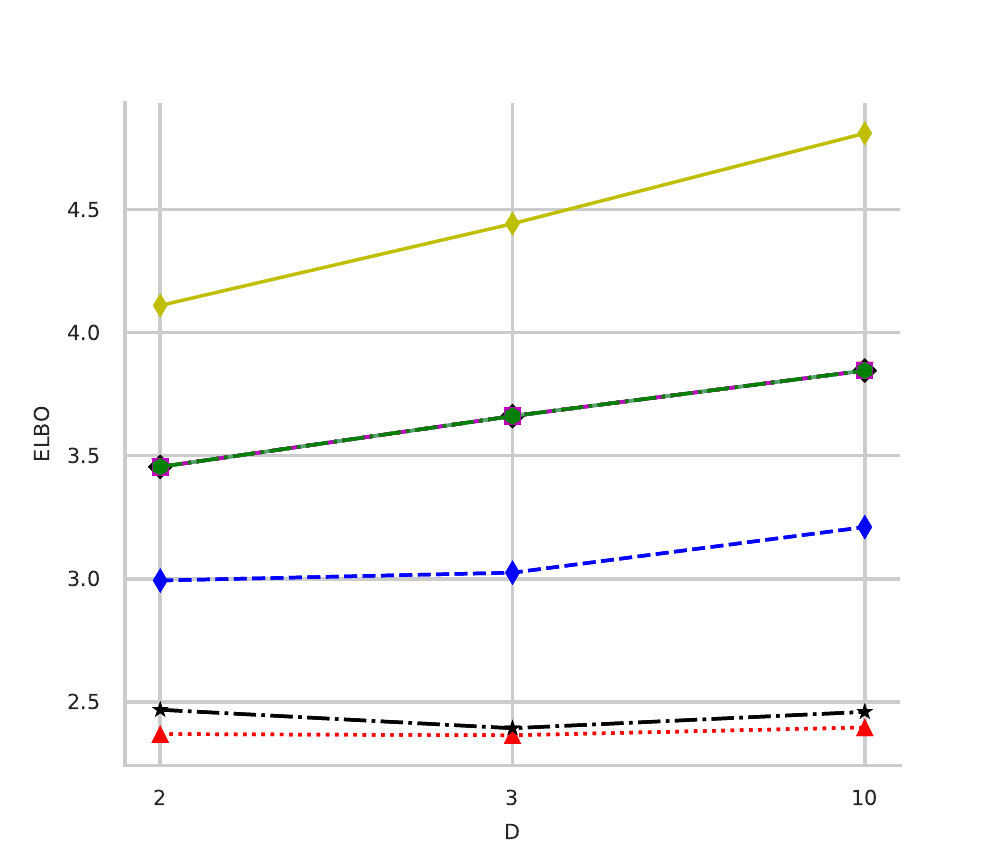}
	}
	\subfloat[$D=10$ \label{fig_}]{%
	\includegraphics[]{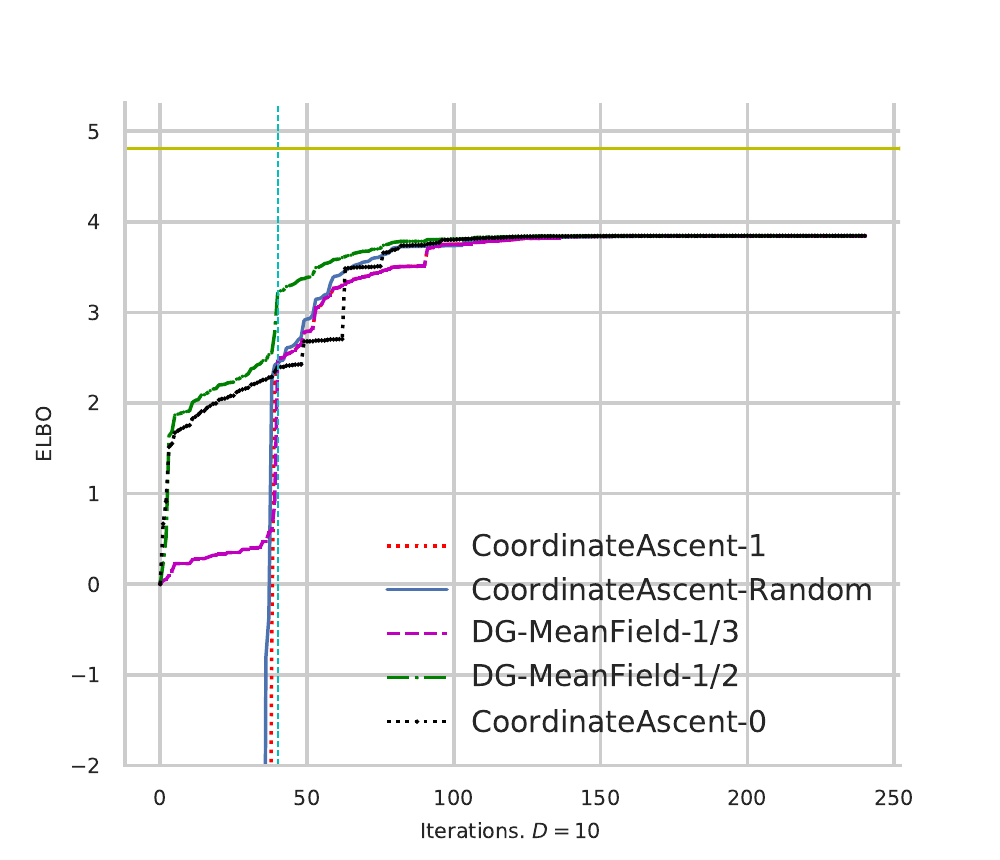}
}\\
	\subfloat[``health", $n=62$ \label{fig_}]{%
	\includegraphics[]{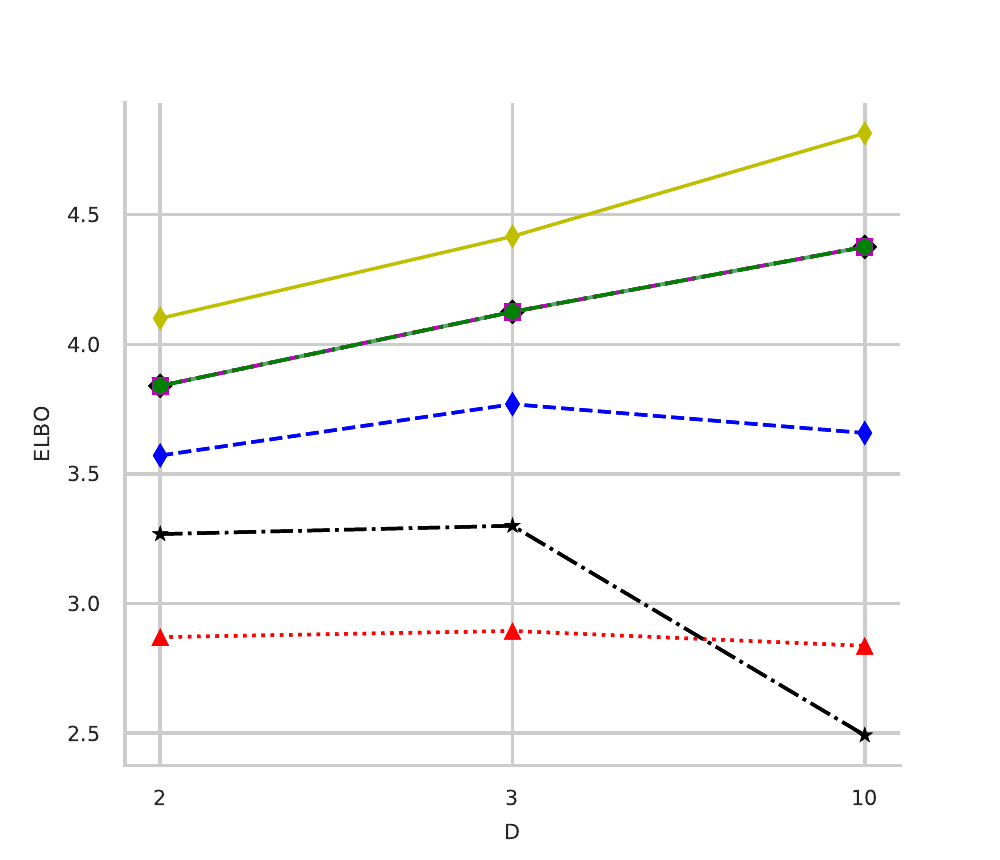}
}
\subfloat[$D=3$ \label{fig_}]{%
	\includegraphics[]{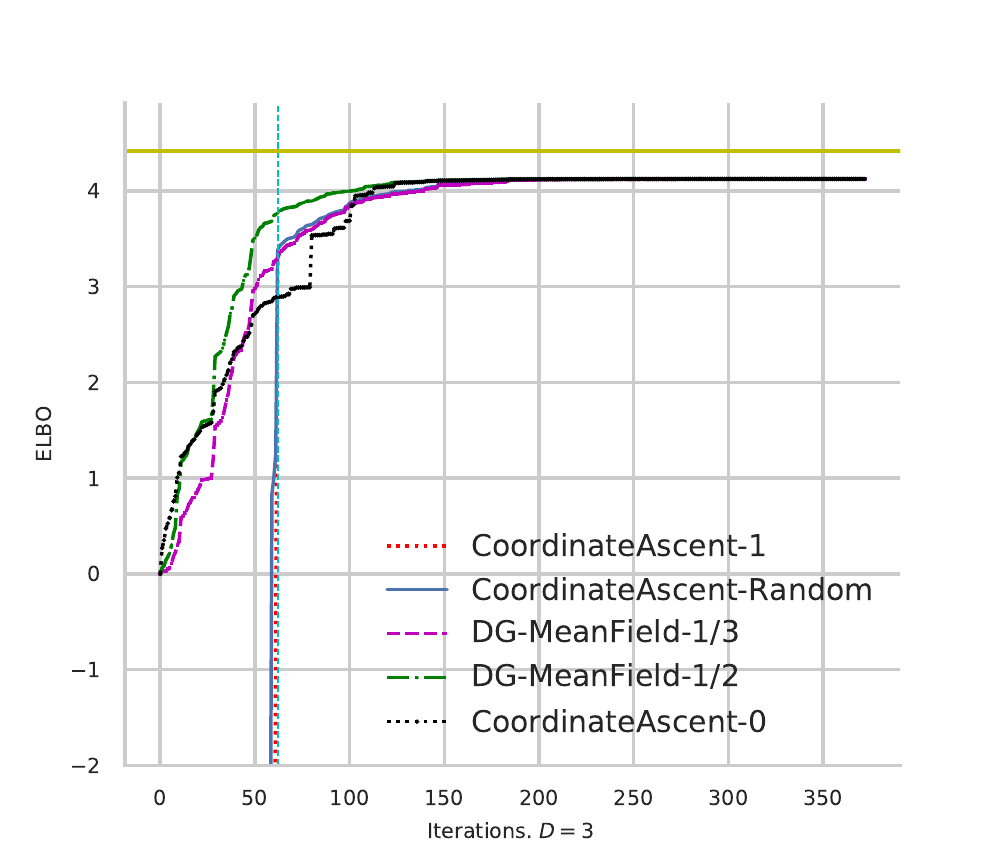}
}
\subfloat[$D=10$ \label{fig_}]{%
	\includegraphics[]{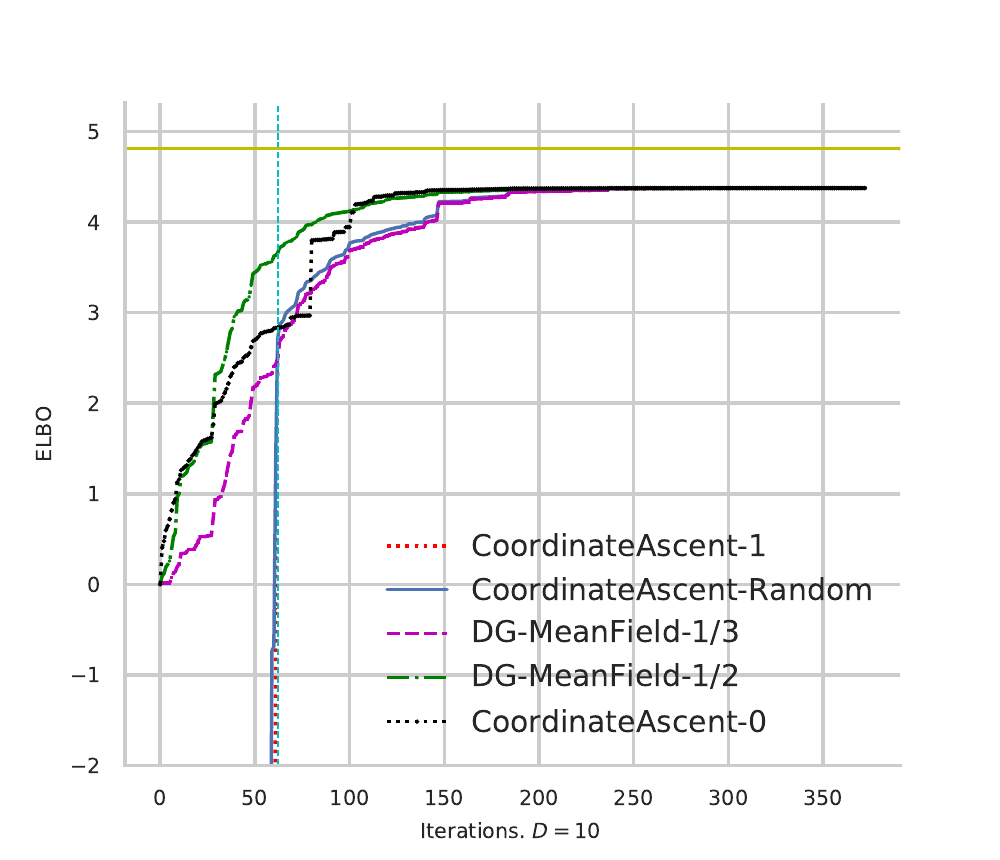}
}\\
	\subfloat[``bath", $n=100$ \label{fig_}]{%
		\includegraphics[]{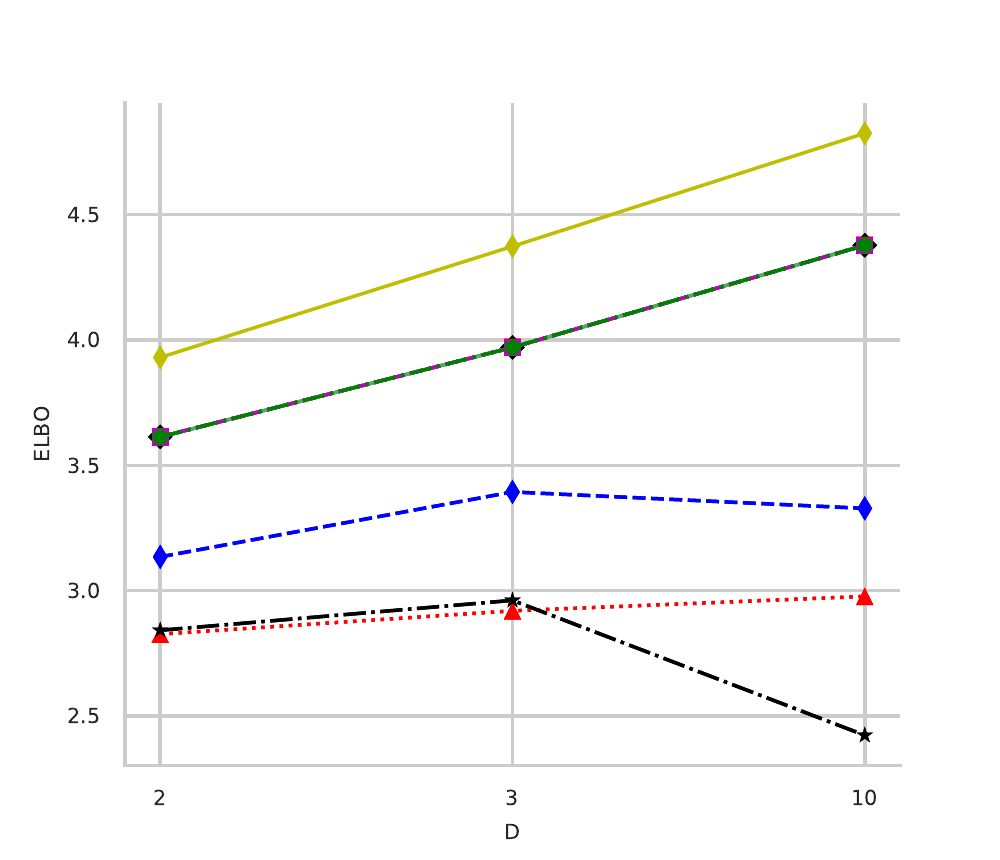}
	}
	\subfloat[$D=3$ \label{fig_}]{%
	\includegraphics[]{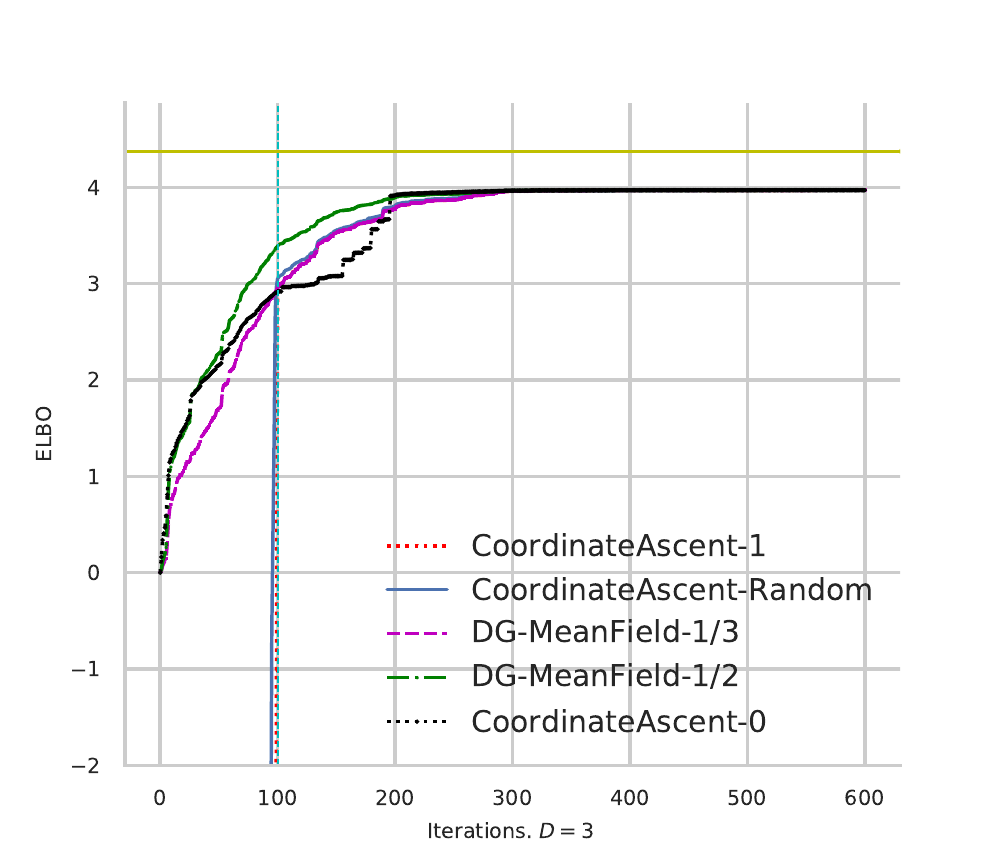}
}
	\subfloat[$D=10$ \label{fig_}]{%
	\includegraphics[]{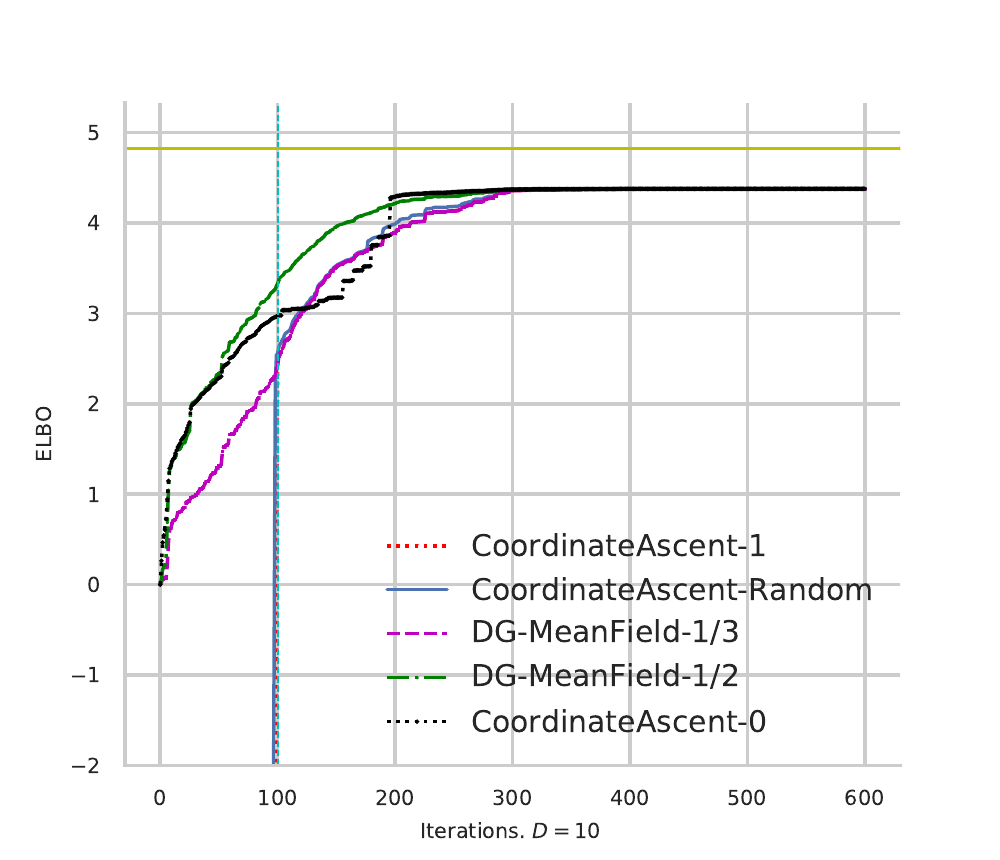}
}\\
	\caption{ELBO objective on Amazon data. 1st row:  ``strollers"; 2nd row: ``health". 3rd row: ``bath"; Subfigs (b,d,g) shows the ELBO returned
	by all algorithms, other columns traces trajectories of multi-epoch algorithms. Cyan vertical line show the one-epoch point. Yellow line shows  the true log-partition.} 
	\label{fig_supp_elbo}
\end{figure}

\subsection{More Results on PA-ELBO Objective}

 \cref{fig_supp_paelbo} illustrates  more results on the  PA-ELBO
objective from Amazon data.

\setkeys{Gin}{width=0.35\textwidth}
\begin{figure}[ht]
	\subfloat[Legend for subfigs (b,d,g) \label{fig_}]{%
	\includegraphics[]{legend_real}
}
	\subfloat[``furniture", $n=40$,  folds (6,10) \label{fig_}]{%
		\includegraphics[]{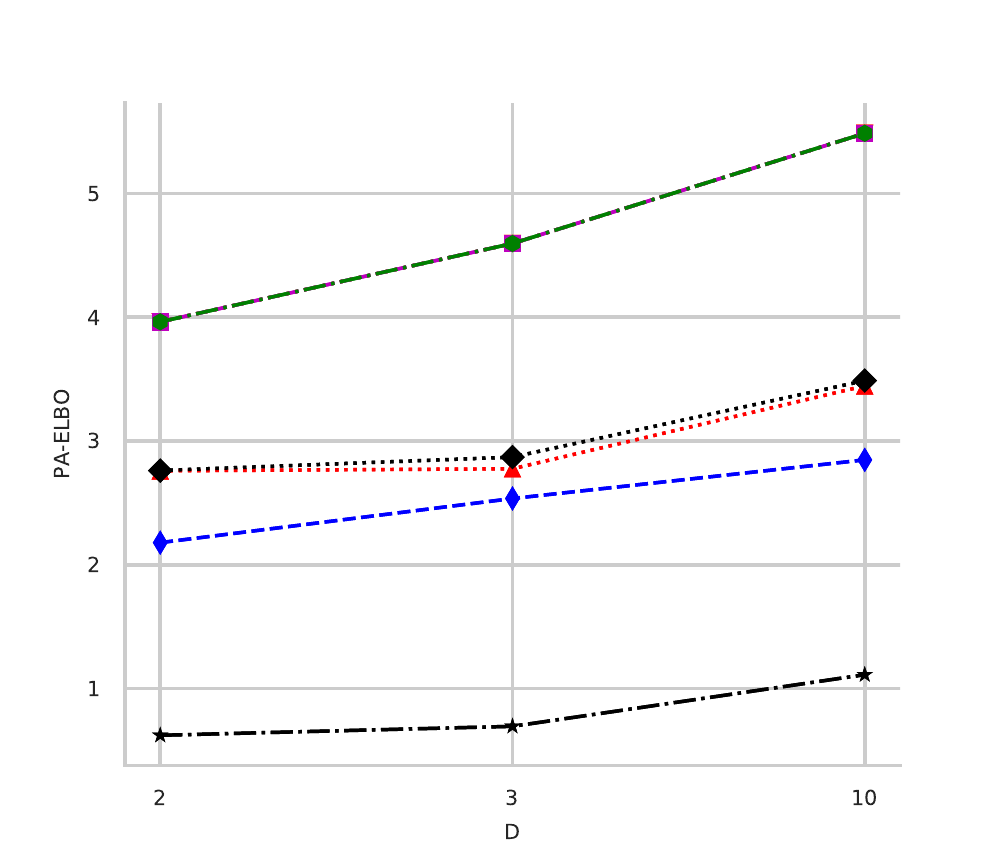}
	}
	\subfloat[$D=10$ \label{fig_}]{%
		\includegraphics[]{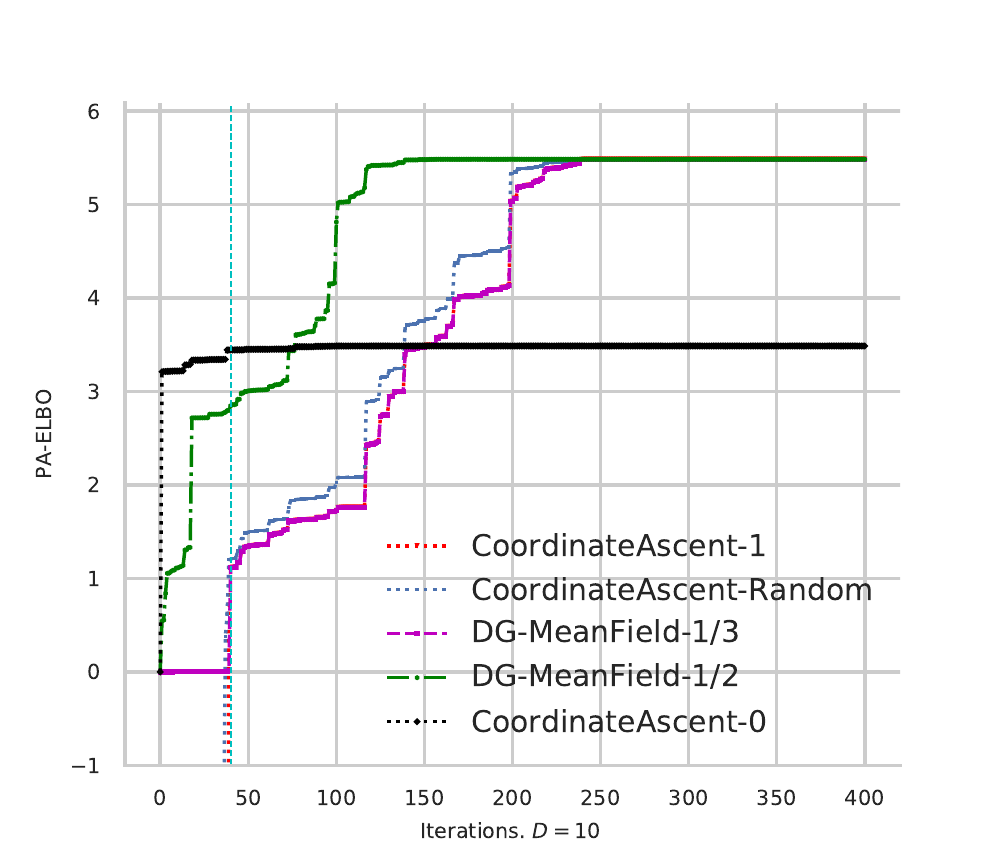}
	}\\
	\subfloat[``toys", $n=62$,  folds (6,8) \label{fig_}]{%
		\includegraphics[]{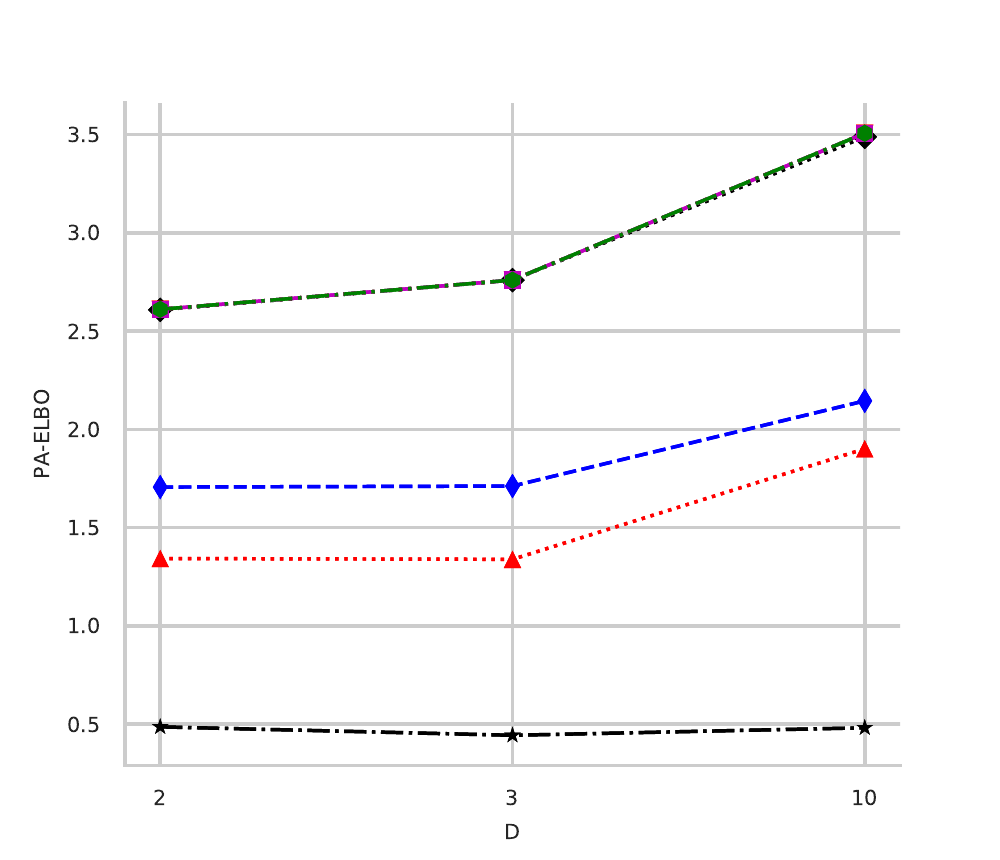}
	}
	\subfloat[$D=3$ \label{fig_}]{%
		\includegraphics[]{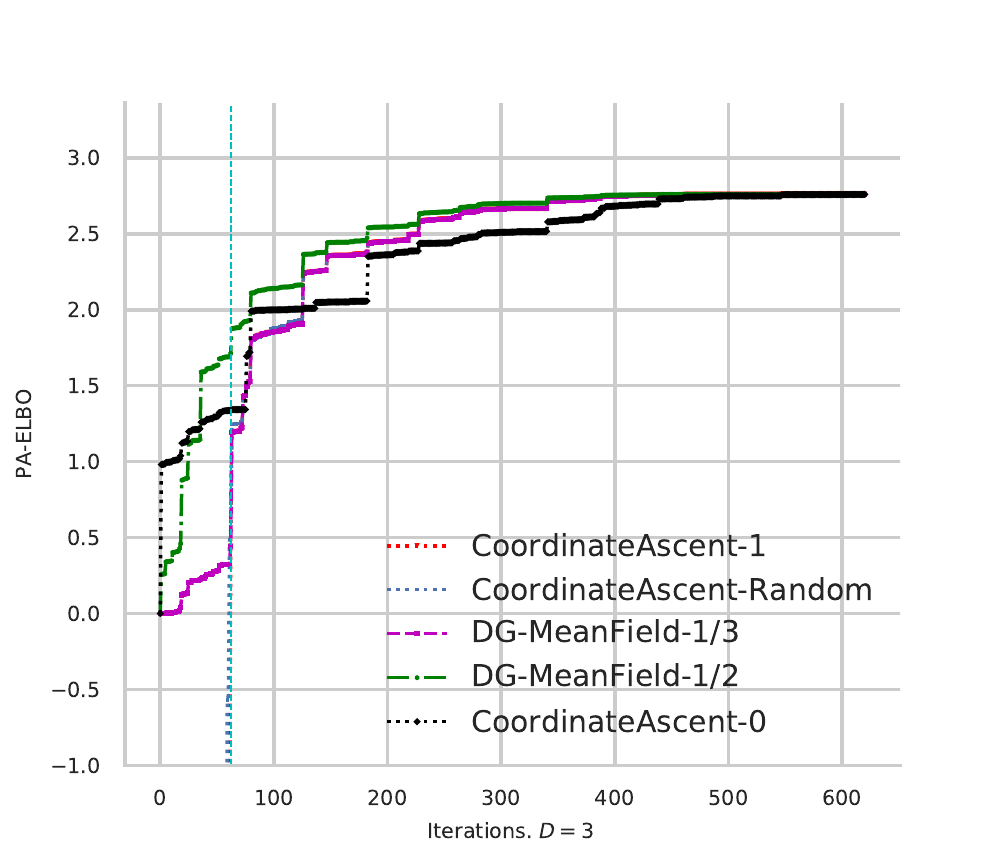}
	}
	\subfloat[$D=10$ \label{fig_}]{%
		\includegraphics[]{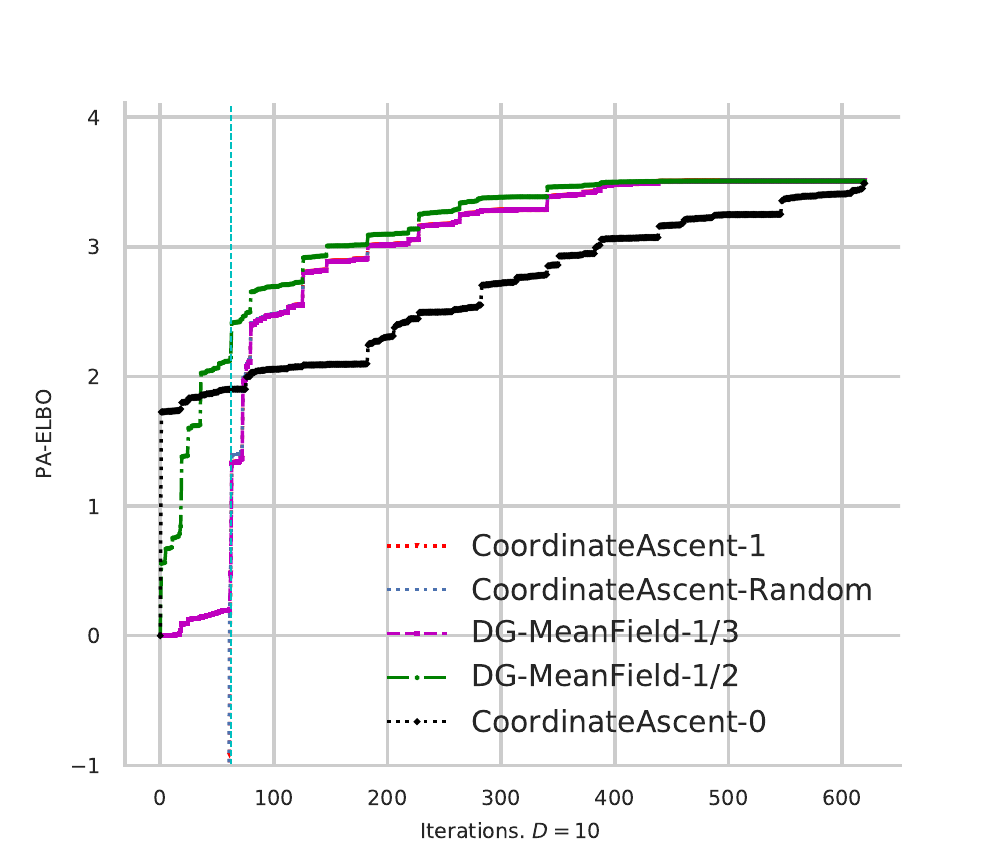}
	}\\
	\subfloat[``bedding", $n=100$,  folds (7,9) \label{fig_}]{%
	\includegraphics[]{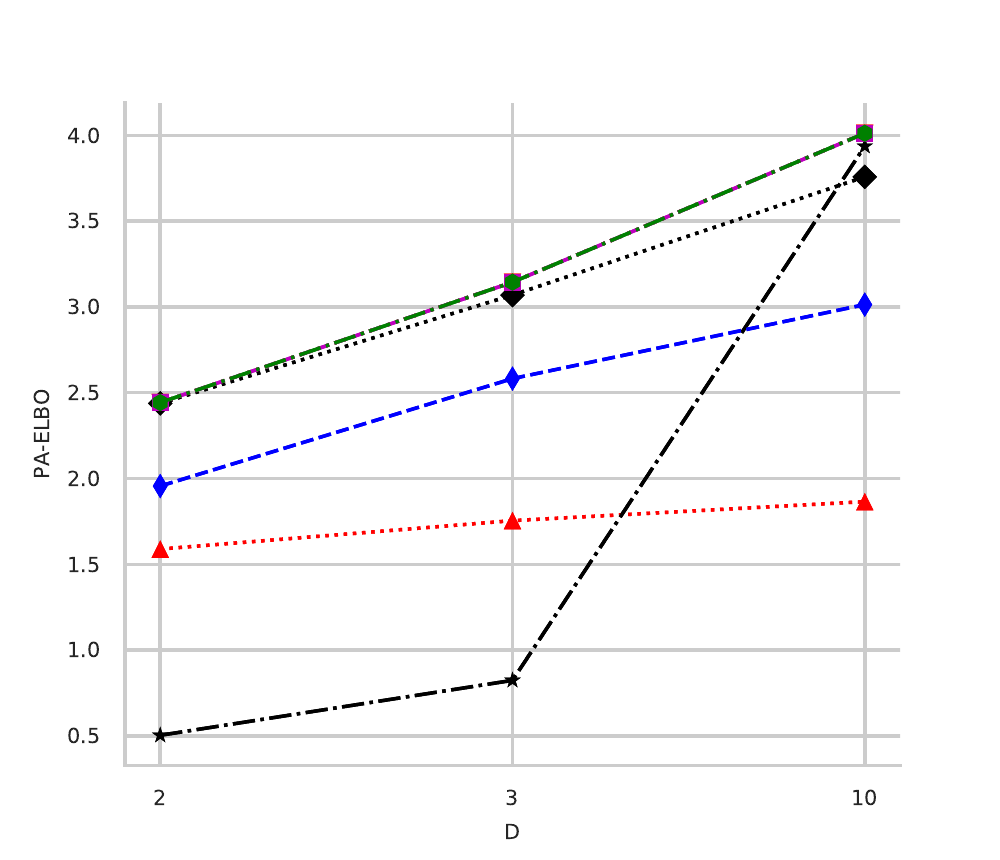}
}
\subfloat[$D=3$ \label{fig_}]{%
	\includegraphics[]{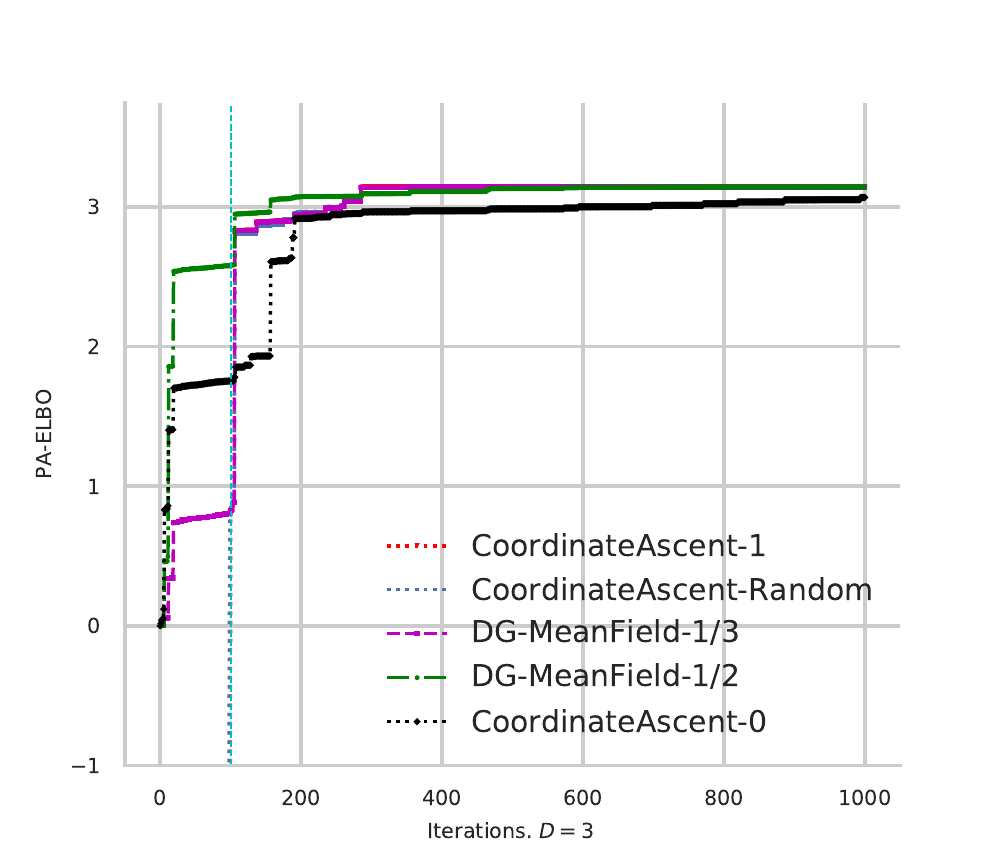}
}
\subfloat[$D=10$ \label{fig_}]{%
	\includegraphics[]{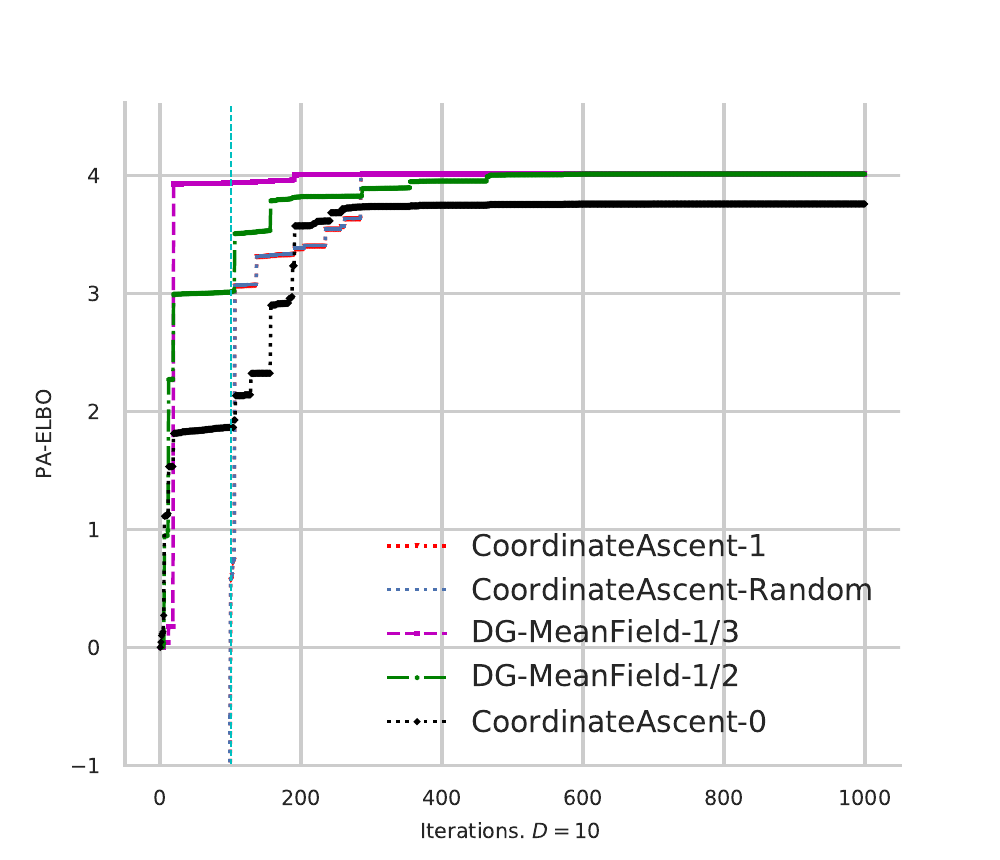}
}\\
	\caption{PA-ELBO objective  on Amazon data. First row: ``furniture"; second row: ``toys"; third row: ``bedding". Subfigs (b,d,g) show the PA-ELBO returned
		by all algorithms, other columns traces trajectories of multi-epoch algorithms. Cyan vertical line shows the one-epoch point.} 
	\label{fig_supp_paelbo}
\end{figure}

\end{document}